%% file: arxiv_v1.tex
\documentclass[]{fairmeta}

\usepackage{xurl}     %

\usepackage{booktabs}
\usepackage{bbm}
\usepackage{url}
\usepackage{amsthm,amssymb,amsmath}
\usepackage{thmtools}
\usepackage{mathtools}
\usepackage{upgreek}
\usepackage{dsfont}
\usepackage{cleveref}
\usepackage{comment}
\usepackage{subcaption}

\usepackage{enumitem}

\input{macros.tex}

\title{Likelihood-Based Reward Designs \\for General LLM Reasoning}

\author[1]{Ariel Kwiatkowski}
\author[2,*]{Natasha Butt}
\author[1]{Ismail Labiad}
\author[1,3,\dagger]{Julia Kempe}
\author[1,\dagger]{Yann Ollivier}

\affiliation[1]{Meta FAIR}
\affiliation[2]{University of Amsterdam}
\affiliation[3]{New York University}

\contribution[*]{Work done during an internship at Meta}
\contribution[\dagger]{Joint senior authors}

\newcommand{\confonly}[1]{}     %
\newif\ifarxiv  \arxivtrue
\newif\ifconference \conferencefalse

\begin{document}

\abstract{\input{abstract2}}

\date{\today}
\correspondence{Ariel Kwiatkowski at \email{ariel.j.kwiatkowski@gmail.com}}

\maketitle

\input{intro.tex}

\input{related.tex}

\input{method.tex}

\input{experiments.tex}

\input{conclusion.tex}

\bibliography{neurips_2025,refs}
\bibliographystyle{iclr2026_conference}

\appendix

\input{proofs.tex}

\end{document}

%% file: macros.tex
\newtheorem{lemma}{Lemma}

\newtheorem{template}{Template}
\declaretheoremstyle[headfont=\bf,bodyfont=\normalfont]{ex}

\declaretheoremstyle[bodyfont=\normalfont]{rm}

\def\E{{\mathbb{E}}} %
\newcommand{\D}{\mathcal{D}}  %

\def\sg{{^\mathrm{sg}}} %
\newcommand{\abs}[1]{\left\lvert#1\right\rvert}  %

\crefname{template}{Template}{Template}

\definecolor{myredd}{RGB}{176,35,24}
\definecolor{mygreenn}{RGB}{78,173,91}
\definecolor{mybluee}{RGB}{46,112,186}
\definecolor{myred}{RGB}{255,225,220}
\definecolor{mygreen}{RGB}{194,244,210} 
\definecolor{myblue}{RGB}{216,236,255}
\definecolor{byzantine}{rgb}{0.74, 0.2, 0.64}
\definecolor{table-blue}{RGB}{173, 216, 230}

\newcommand{\highlightbox}[1]{
\begin{tcolorbox}[
  colback=myblue!50!white,  %
  boxrule=0pt,          %
  frame empty,
  boxrule=0pt,
  colframe=white,
  width=\textwidth,     %
  left=3pt,             %
  right=0pt             %
]
#1
\end{tcolorbox}
}

%% file: abstract2.tex
Fine-tuning large language models (LLMs) on reasoning benchmarks via
reinforcement learning requires a specific reward
function, often binary, for each benchmark.  This comes with two
potential limitations: the need to design the reward, and the potentially
sparse nature of binary rewards. 
Here, we systematically
investigate rewards derived
from the probability or log-probability of emitting the reference answer
(or any other prompt continuation present in the data), which have the advantage of not relying on specific
verifiers and being available at scale.
Several recent
works have advocated for the use of similar rewards (e.g., VeriFree, JEPO,
RLPR, NOVER).

We systematically compare variants of likelihood-based rewards with standard
baselines, testing performance both on standard mathematical reasoning
benchmarks, and on long-form answers where no external verifier is
available. We find that using the \emph{log-probability} of the reference
answer as the reward for chain-of-thought (CoT) learning is the only
option that performs well in all setups.  This reward is also consistent with
the next-token log-likelihood loss used during pretraining.

In verifiable settings, log-probability rewards bring comparable or
better success rates than reinforcing with standard binary rewards, and
yield much better perplexity. In non-verifiable settings, they perform
on par with SFT. On the other hand, methods based on probability, such as VeriFree, flatline on
non-verifiable settings due to vanishing probabilities of getting the
correct answer.

Overall, this establishes log-probability rewards as a viable method for
CoT fine-tuning, bridging the short, verifiable and long,
non-verifiable answer settings. 

%% file: intro.tex
\section{Introduction}

Large language models (LLMs) have achieved striking progress on tasks requiring reasoning, from mathematics to code generation \citep{cobbe2021training,HendrycksBKABTS21,openai2023gpt4}. A central ingredient has been chain-of-thought (CoT) prompting, where models articulate intermediate reasoning steps before producing a final answer \citep{wei2022chain,guo2025deepseek}. However, CoTs are rarely available in raw training data, making reinforcement learning (RL) the predominant approach: the CoT is treated as a sequence of actions, and correctness of the final answer determines the reward. This paradigm works well in verifiable domains such as mathematics and programming, where ground-truth correctness is available \citep{cobbe2021training,HendrycksBKABTS21,chen2021evaluating,austin2021programs,hendrycks2021measuring_apps}, but it does not naturally extend to non-verifiable domains like long-form proofs or open-ended generation.

To overcome this limitation, we investigate reusing training signals
closer to the log-likelihood signal already employed during pretraining.
Instead of sampling answers and relying on 0/1 correctness rewards, we
reward the model for increasing the probability or log-probability of
the answers present in the training data.
Such criteria are universal—they apply in both verifiable and
non-verifiable settings and could provide a denser signal.
Such approaches are present, e.g., in \citet{zhou2025verifree} (training
with the probability of the reference answer) or
\citet{tang2025beyond} (training with a
variant of the log-probability of the reference answer).
Note that reference answers are available for situations
in which 0/1 rewards are not, such as long-form question-answering. This
makes it possible to test these methods both in verifiable and
non-verifiable, long-form answer settings.
We particularly focus on the case of log-probability, since it is
conceptually the closest to the pretraining criterion.

\paragraph{Our Approach and Contributions.}
We conduct the first comprehensive study of probability-based RL rewards
for CoT training, spanning verifiable and non-verifiable domains, across multiple model families (Qwen-2.5, Llama-3.2).
Our main contributions and findings are:

\begin{itemize}%
\item {\bf Systematic evaluation across domains.} We test many
variants of probability-based rewards (probabilities and
log-probabilities, including several variants from the literature
such as VeriFree, RLPR, JEPO) for CoT training, comparing against supervised
fine-tuning (SFT) and standard RL training (RLOO) baselines. We run
the comparisons on
two verifiable benchmarks – MATH
\citep{HendrycksBKABTS21}, DeepScaleR \citep{deepscaler2025} –
and two non-verifiable settings - Alpaca \citep{alpaca} and the
non-verifiable "proof portion" of NuminaMath
\citep{numina_math_datasets}.

\item {\bf Universality of log-probability rewards.} Among the variants
tested, rewards based on \emph{log}-probabilities 
perform well in every scenario (short, verifiable answers and long,
non-verifiable answers), while all others fail in one or several
settings.

\item {\bf Advantages of probability-based rewards.} For verifiable
domains, all variants of
probability-based rewards perform similarly and slightly outperform base RL training in terms
of greedy success rate on verifiable domains. They also offer some
computational advantages during training (no need to sample an answer).

\item {\bf Success and perplexity trade-offs.} In verifiable domains,
log-probability rewards perform well both in terms of success rate and of
perplexity -- a key metric aligned with
pretraining. On the other hand, both base RL training and
probability-based rewards
perform extremely poorly on perplexity (much worse than SFT). This highlights a distinct advantage of log-probabilites.

\item {\bf Non-verifiable domain behavior.} On long-form domains, both
base RL and pure probability rewards collapse due to vanishing
probabilities of long answers. Log-probability rewards remain viable and
perform similarly to SFT.

\item {\bf CoT shortening with log-probability rewards.} 
In every scenario, log-probability rewards lead to an initial shortening
of the CoT. For verifiable domains, the length of the CoT recovers during
training.
On the other hand, for non-verifiable domains, the CoT stays
very short, meaning log-probability rewards largely follow SFT from that
point. On verifiable domains, base RL and pure probability rewards
(VeriFree) do not exhibit this shortening.
Mitigating
strategies such as CoT length rewards and KL penalties maintain CoT 
but hurt performance. Thus, it seems that RL CoT training on non-verifiable
domains can only match SFT by eliminating the CoT. We discuss hypotheses
around this phenomenon. %

\end{itemize}

Overall, these results establish log-likelihood rewards as a simple way
to bridge verifiable and non-verifiable settings under a single training
criterion, broadly applicable for fine-tuning  LLMs. %

%% file: related.tex
\paragraph{Related Work.} Several prior works have proposed to modify the binary rewards in standard RL post-training settings. We can globally distinguish these rewards into {\em intrinsic} rewards that do not require ground-truth, and those that use the {\em confidence or log-likelihood of the ground-truth} answer. The former category utilizes measures of confidence, entropy or diversity as measured by the generating language model itself
\citep{prabhudesai2025maximizingconfidenceimprovesreasoning,agarwal2025the,zhao2025learningreasonexternalrewards,li2025confidenceneedfewshotrl,gao2025one}.
Nevertheless, these intrinsic rewards generally cannot surpass rewards grounded in true correctness except under strong coverage assumptions, and tend to lead to reward hacking or diversity collapse. \citet{huang2025sharpening} show that self-rewarding can only “sharpen” knowledge already covered by the base model—it cannot create new information—so performance is bounded by model coverage (i.e. its pass@k rate). \citet{song2024mind} formalizes the generation–verification gap and shows self-improvement hinges on sufficient coverage and verifier quality; when these are weak, intrinsic/self-verification stalls and fails to match correctness-based training. 
Finally, \citet{huang2025bestofn} proves that inference-time alignment with imperfect reward models suffers reward hacking and lacks guarantees under realistic coverage, again falling short of what verified rewards can achieve. Another study \citep{kayal2025intrinsic}, shows that certain intrinsic signals (like policy entropy or state novelty) can fail in high-dimensional or complex output spaces, and sometimes result in exploration that
diverges from the downstream task. Some works combine intrinsic and binary rewards 
\citep{song2025outcomebased,li2025darling} to encourage exploration. 
Yet another line of works explores using LLM-as-a-judge synthetic rewards in RL-based post-training (RLAIF), explored as an alternative to human feedback \citep{lee2024rlaif,bai2022constitutionalaiharmlessnessai} or for (semi-)verifiable domains \citep{whitehouse2025j1,jayalath2025computeteacherturninginference,simonds2025rlsrreinforcementlearningself}.

Closer to our line of work are works that use the probability or log-likelihood of the reference answer given a generated reasoning chain under the initial policy model to provide a verifier-free scoring function.
We highlight the works relevant to our setting and label them with distinctive keywords for clarity. To the best of our knowledge, none of these studies investigate log-likelihoods as a primary reward signal, with the exception of \citet{tang2025beyond}, who include it as an ablation against their proposed JEPO reward and report weaker performance. In contrast, we introduce log-likelihood rewards as a primary training signal, and our experiments consistently demonstrate their competitiveness across models and datasets, including those evaluated in prior work.

\begin{itemize}
    \item {\em VeriFree} \citep{zhou2025verifree} uses probabilities of reference answers as reward in verifiable domains
    \item  {\em JEPO} \citep{tang2025beyond} introduces  
    a Jensen-based ELBO loss with log-probs. In experiments they
    mix verifiable with non-verifiable data to show that the verifiable part improves with this loss.
    \item  {\em RLPR} \citep{yu2025rlpr} uses average probability of the ground truth for non-verifiable domains.
    \item {\em NOVER} \citep{liu2025noverincentivetraininglanguage} is a variant of probability-based rewards, using a geometric mean of per-token perplexities.
    \item {\em Reinforcement-pretraining}
    \citep{dong2025reinforcementpretraining} performs small-scale {\em
    pretraining} from scratch, inserting CoTs at specific points and
    rewarding for correct continuation over a few tokens.%
\item {\em LongForm} \citep{gurung2025learning} designs a clever reward function (VR-CLI) that allows them to use an unlabeled book dataset as a learning signal for reasoning.
\end{itemize}

%% file: method.tex
\section{Method}
\label{sec:method}

\paragraph{Context: Chain-of-thought fine-tuning via Reinforcement
Learning.} We consider the general context of fine-tuning an LLM to
improve performance on a set of questions-answers via a Chain-of-Thought
(CoT) optimized by reinforcement learning. For each prompt $p$, the
fine-tuned model should first print a CoT $z$, then an answer $a$. Then a
reward $R$ is computed depending on $a$ (such as correctness, or matching
some reference answer). Fine-tuning should optimize the expected reward.

Denoting $\pi_\theta$ the generative probabilistic model with
parameter $\theta$, and $\D$ the dataset (a distribution of questions or
prompts $p$), we want to maximize
\begin{equation}
J_\theta= \E_{p\sim \D}\, \E_{z\sim \pi_\theta(z|p),\,a\sim
\pi_\theta(a|p,z)} [R(z,a)]
\end{equation}
where $R(z,a)$ is the reward obtained for CoT $z$ and answer $a$.

This task is often tackled with RL variants of the basic Reinforce
algorithms, such as RLOO \citep{ahmadian2024back}, GRPO \citep{guo2025deepseek}, or PPO \citep{schulman2017}. 

\paragraph{RL fine-tuning with probability-based rewards.} We focus on the case when a reference
answer $a^\star$ is available for each prompt in the dataset. Then it is
possible to estimate the probability of this answer given the CoT. We
will compare RL training with several rewards derived in this setting.

For
instance, we can set a reward similar to the log-loss used during
pretraining, 
\begin{equation}
R(z,a)=\log \pi_\theta (a^\star|p,z).
\end{equation}
We call this setting
\emph{log-prob rewards}.
Given a CoT $z$, this quantity can be computed in one pass of a
transformer on the reference answer $a^\star$.
In particular, since the reward depends on $z$ and $a^\star$ but not on $a$, sampling
of an answer $a$ given the CoT $z$ is not necessary.

We also consider the \emph{average log-prob reward} variant
\begin{equation}
R(z,a)=\frac{1}{\abs{a^\star}} \log \pi_\theta (a^\star|p,z)
\end{equation}
namely, we compute the per-token log-probability by downscaling the
reward by the length $\abs{a^\star}$ of the answer. This results in a different weighting of the various data samples in the dataset.

Log-prob rewards are aligned with the pretraining phase of
LLM training, where the criterion is the log-probability of the next
token. They do not require access to a verifier, only to a
reference answer (or any continuation) in the data. Thus, they can
potentially be applied any question-answer pairs.

The logprob reward setting is also considered in \cite{tang2025beyond}, although they largely focus on a ``multi-sample'' variant. The
gradient of the expected reward is derived there as
\begin{equation}
\nabla J_\theta=\E_{p\sim \D}\,\E_{z\sim \pi_\theta(z|p),\,a\sim \pi_\theta(a|p,z)}\!\left[\log \pi_\theta(a^\star|p,z)\,\nabla \log \pi_\theta(z|p)+\nabla \log \pi_\theta(a^\star \mid p,z)\right]
\end{equation}
As noted in \citet{tang2025beyond}, the second term is analogous to a
supervised fine-tuning term that directly optimizes the log-likelihood of
the reference answer $a^\star$ given what comes before, and the first
term is a traditional Reinforce term with reward $\log
\pi_\theta(a^\star|p,z)$. For completeness, we derive this gradient in
\Cref{app:losses}, together with its application to RL algorithms such as
RLOO.

A related but different reward appears in \citet{zhou2025verifree}:
\begin{equation}
R_{\text{VeriFree}}(z,a)=\pi_\theta(a^\star|p,z)=\E_{a\sim
\pi_\theta(a|p,z)} [1_{a=a^\star}]
\end{equation}
thus, without the logarithm. This is the \emph{expected} success rate for
matching the reference answer $a^\star$: \emph{in expectation}, it is the
same as using binary rewards, namely, sampling an answer $a$ given the CoT, and setting a reward $1$ if
$a=a^\star$.
\citet{zhou2025verifree} prove that working with the expectation reduces
variance compared to sampling $a$, and this affects training dynamics.

The VeriFree reward diverges from logprob rewards when probabilities are
very small. For instance, if initially the model has an almost-zero
probability to reach the reference answer, then the VeriFree reward
produces no learning. Similarly, for long free-form answers, the
probability of an exact match with $a^\star$ is tiny, so we would expect a
difference between VeriFree and logprob rewards. On the other hand, if 
the initial probability to reach the correct answer is reasonably high,
then we expect the VeriFree and logprob rewards to be well aligned.

\paragraph{Algorithms and rewards tested.} We now give an outline of the algorithms
compared in the experiments.

For every RL algorithm except JEPO, the advantages used for the Reinforce
gradient updates are obtained by RLOO, i.e., by subtracting from the
reward a leave-one-out estimate of the mean reward estimated on a
minibatch for a given prompt; this is an unbiased version of GRPO
\citep{guo2025deepseek}.

\begin{itemize}%
\item \emph{SFT}: standard fine-tuning with the next-token cross-entropy loss.
Namely, we
omit the CoT, and fine-tune the model to predict the ground truth
directly from the prompt.

\item \emph{Base RL}: this is the most direct RL method. For each prompt $p$, we
sample a CoT $z\sim \pi_\theta(z|p)$, then an answer $a\sim
\pi_\theta(a|p,z)$, and check whether the answer is correct:
\begin{equation}
R_\text{RLOO}(z,a)=1_{a=a^\star}
\label{eq:RLOO}
\end{equation}
given the reference answer $a^\star$. As for all other RL methods, we
employ a leave-one-out advantage estimation (RLOO).

\item \emph{Probability} (VeriFree): As mentioned above, the reward is
\begin{equation}
R_{\text{Probability}}(z,a)=\pi_\theta(a^\star|p,z)=\E_{a\sim
\pi_\theta(a|p,z)} [1_{a=a^\star}]
\label{eq:verifree}
\end{equation}
namely, instead of sampling an answer $a$ from the model, we directly
compute the probability of the reference answer $a^\star$ given $z$ using
the model $\pi_\theta$.

\item \emph{Average prob} (AvgProb): Similarly to RLPR \citep{yu2025rlpr}, the reward is
set to the \emph{average per-token probabilities} of the reference
answer:
\begin{equation}
R_{\text{avgprob}}(z,a)=\frac{1}{\abs{a^\star}}\sum_{t=1}^{\abs{a^\star}}
\pi_\theta(a^\star_t|p,z,a^\star_{[1:t-1]})
\label{eq:avg_prob}
\end{equation}

\item \emph{Log-prob}: the reward is
\begin{equation}
R_{\text{log-prob}}(z,a)=\log \pi_\theta(a^\star|p,z)
\label{eq:logprob}
\end{equation}
namely, we directly compute the log-likelihood of the reference answer
$a^\star$ given $z$.

\item \emph{Average log-prob} (AvgLogprob): In log-probs, longer answers have rewards of a
bigger magnitude, since $\log \pi_\theta(a^\star|p,z)$ is a sum over all
tokens in $a^\star$. Average log-probs rescales the reward accordingly:
\begin{equation}
R_{\text{avglogprob}}(z,a)=\frac{1}{\abs{a^\star}}\log \pi_\theta(a^\star|p,z)
\label{eq:avg_logprobs}
\end{equation}
where $\abs{a^\star}$ is the number of tokens in $a^\star$. Compared to
log-probs, this just means that different answers in the dataset are
weighted in a different way.

\item \emph{JEPO} \citep{tang2025beyond} used a
refined version of the group reward in GRPO and RLOO, by noting that the
expected log-probability $\E_{z\sim \pi_\theta(z|p)} \log
\pi_\theta(a^\star|p,z)$ is an underestimate of the actual
log of the probability to get $a^\star$ using $\pi_\theta$, which is
$\log \E_{z\sim \pi_\theta(z|p)}\pi_\theta(a^\star|p,z)$. So, starting
from GRPO, they introduce a group-level reward based on $G$ samples
$z_1,\ldots,z_G$ for a given prompt,
\begin{equation}
R(z_1,\ldots,z_G)=\log \frac{1}{G}\sum_{i=1}^G \pi_\theta(a^\star|p,z_i).
\label{eq:JEPO}
\end{equation}
Compared to log-probs over a similar minibatch $z_i$, the reward is the
log-mean-exp of rewards in the minibatch. For Reinforce advantage
estimation,
they subtract the similar estimate over $G-1$ samples without the sample
$z_i$. We will use $G=4$ as in
\citet{tang2025beyond}. %

\end{itemize}

\paragraph{Success metrics.} For each algorithm, we report several
success metrics. These metrics largely follow the quantities tracked by
the different algorithms.

We denote by $\D$ the distribution of
prompts and reference answers in the dataset.

\newcommand{\CoT}{^\text{CoT}}

Given a prompt $p$, the probability to obtain the correct answer using a
CoT model $\pi$ is
\begin{equation}
\pi\CoT(a^\star|p)=\E_{z\sim \pi(z|p)} \left[
\pi(a^\star|p,z)
\right].
\end{equation}

\begin{itemize}%
\item \emph{Success rate}: This is the probability to get a correct answer,
averaged over the dataset,
\begin{equation}
\E_{(p,a^\star)\sim \D} \left[
\pi\CoT(a^\star|p)\right].
\label{eq:success}
\end{equation}
It can be estimated directly by sampling a prompt and answer in the
dataset, sampling a CoT $z$, and computing $\pi_\theta(a^\star|p,z)$.
This is the estimate we report.
VeriFree and Base RL directly optimize the success rate. We consider two modes
for generating the answers given a prompt and chain of thought:
\emph{Greedy success}, where the most likely token is used at each step,
and \emph{$T=1$ sampling success} from the softmax probabilities at
temperature $T=1$.

\item \emph{Log-probability}: This is a family of metrics that aggregate the likelihood of answer tokens across the dataset,
\begin{equation}
\E_{(p,a^\star)\sim \D} \left[\log \pi\CoT(a^\star|p)\right].
\end{equation}

To keep these quantities comparable, we consider two averaging schemes – {\em per-token} and {\em per-answer}.

\begin{itemize}
    \item \emph{Per-token log-probabilities} sums the log-probabilities of all answer tokens in the dataset, and divides by the total number of those tokens. Equivalently
\begin{equation}
\frac{1}{\E_{(p,a^\star)\in \D}
[\abs{a^\star}]} \, \E_{(p,a^\star)\in \D} \left[
\log \pi\CoT(a^\star|p)
\right]
\end{equation}

\item \emph{Per-answer log-probabilities} averages across each answer, then averages over the dataset:
\begin{equation}
\E_{(p,a^\star)\in \D}
\left[
\frac{1}{\abs{a^\star}} \log \pi\CoT(a^\star|p)
\right]
\end{equation}

\end{itemize}

However, these metrics are difficult to estimate directly due to the expectation over $z$ inside the $\log$, since $\pi\CoT$ is an expectation.
A simple solution is to estimate the average via Monte Carlo using $N$
samples \citep{tang2025beyond}:
\begin{equation}
\E_{(p,a^\star)\sim \D} \left[
\log \frac{1}{N} \sum_{i=1}^N \pi(a^\star|p,z_i)
\right]
\end{equation}
where the $z_i$ are sampled from $\pi(z_i|p)$. We refer to this as
\emph{logprob-MC$N$}. We apply this modification to both per-answer and per-token averaged logprobs.

We will use both the ``naive'' estimate \emph{logprob-MC1} with $N=1$, and a more precise estimate, {\em logprob-MC32}, computed less frequently during training.

For supervised fine-tuning (SFT) with no CoT, this is irrelevant as there is no expectation over $z$, and we can report $\log \pi(a^\star|p)$ directly.

The MC estimate is always an \emph{underestimate} of the actual logprob $\log
\pi\CoT(a^\star|p)$, since $\log$ is concave. This should be kept in mind
when comparing logprob-MC1 to SFT log-probabilities.

\item We also report \emph{perplexity}, which is just the exponential of
minus per-answer log-probabilities. Technically, this corresponds to {\em per answer perplexity-MC1}; we shorten to {\em perplexity}. This is also equal to the geometric
mean of the perplexity of the answer for each prompt in the dataset. 

\item \emph{Average CoT length}: We also report the average length of the
CoTs used by a model,
\begin{equation}
\E_{(p,a^\star)\in \D} \,\E_{z\sim \pi(z|p)} \left[\abs{z}\right]
\end{equation}
as a relevant quantity for analysis. Note that this includes formatting
tokens.

\end{itemize}

%% file: experiments.tex
\section{Experimental Results}

\subsection{Setup: Datasets, Models, and Protocol}

{\bf Models.}
We evaluate on two instruction-tuned models:
\textsc{Llama-3.2-3B-Instruct} ~\citep{dubey2024llama},
and \textsc{Qwen-2.5-3B-Instruct}~\citep{Yang2024Qwen25TR}.

{\bf Datasets.}
We consider two \emph{verifiable} math benchmarks and two \emph{non-verifiable} long-form datasets.
(i) \textbf{MATH}~\citep{HendrycksBKABTS21}:We report accuracy on the official test split. The resulting training set contains $\sim$7{,}000 short-answer problems.
(ii) \textbf{DeepScaleR (Preview)}~\citep{deepscaler2025}: we hold out a random $10\%$ for validation to report performance. The training set has $\sim$39{,}000 short-answer problems.
(iii) \textbf{Alpaca (cleaned)}~\citep{alpaca}: we use the standard cleaned variant; $1{,}000$ random examples are used for validation, leaving $\sim$50{,}000 training samples with predominantly long-form answers.
(iv) \textbf{NuminaProof}: starting from \textsc{NuminaMath-1.5}~\citep{numina_math_datasets}, we filter for theorem–proof style items. We reserve $1{,}000$ examples for validation, yielding $\sim$50{,}000 long-form training samples. More detail in \Cref{app:expdetail}.

{\bf Algorithms tested.} 
We compare the algorithms mentioned in Section~\ref{sec:method}, namely, SFT and the following RL variants: Base
RL, Probability (VeriFree), Logprob, AvgLogprob, AvgProb, and JEPO. These
differ by the rewards used, as described in Section~\ref{sec:method}. Details in \Cref{app:expdetail}.

{\em Verifiable.} We run experiments with all methods on verifiable
domains with RLOO group size $G=4$ and $G=32$, except for JEPO, where we only run with group size
$G=4$ (the value used in
\citet{tang2025beyond}) as JEPO is harder
to implement efficiently for larger $G$\footnote{Because the JEPO
reward depends on the whole group and cannot be computed for each sample
independently, efficient implementation with large $G$ is more delicate.}.
In the loss function we include a KL divergence regularization term as proposed by ~\citet{guo2025deepseek} with a coefficient of 0.001.

{\em Non-verifiable.} In non-verifiable domains, we run with $G=4$
throughout. Here, we do not use a KL divergence term in the main results,
but we explore its impact %
in the ablations.

\subsection{Results on Verifiable Domains}

\input{tables/verifiable_g32}

We present the results on verifiable domains in Table
\ref{tab:verifiable_g32}, and
\Cref{fig:llama_math_g32,fig-app:qwen_math_g32,fig-app:llama_ds_g32,fig-app:qwen_ds_g32}
in \Cref{app:experiments} for $G=32$. The results for $G=4$ (including
JEPO) appear in \Cref{tab:verifiable_g4} and \Cref{fig-app:llama_math_g4,fig-app:qwen_math_g4,fig-app:llama_ds_g4,fig-app:qwen_ds_g4} in \Cref{app:experiments}.

The key takeaway is that 
\emph{all RL variants based on ground-truth answers have similar success
rates}
for greedily decoded answers. More precisely, all (log-)probability-based
variants perform better than Base RL when run with standard group size $G=32$.

\begin{figure}[h]
    \centering
    \includegraphics[width=\linewidth]{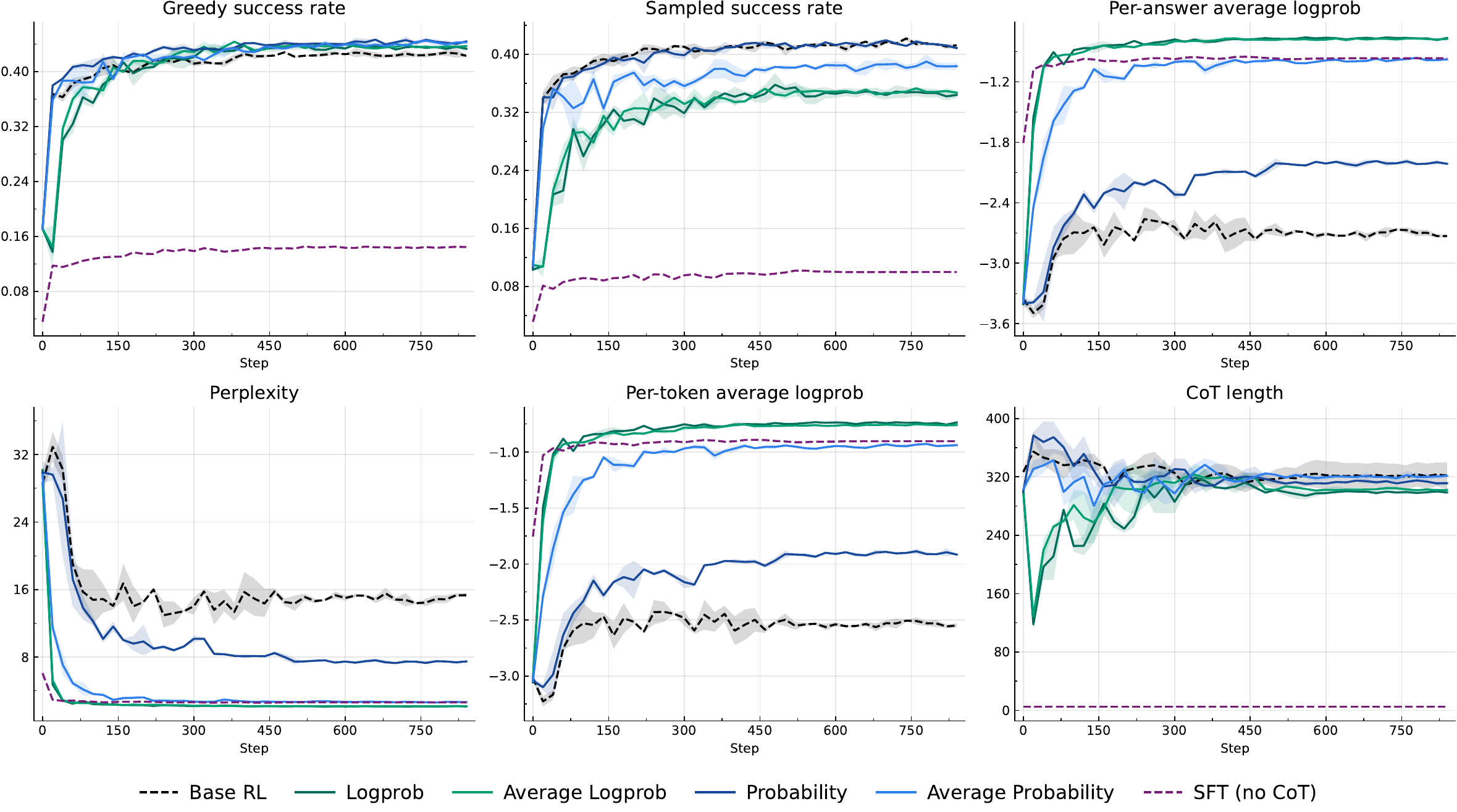}
    \caption{\small{\bf Verifiable. Llama 3.2 3B Instruct on MATH, G=32.} Learning curves of our algorithms for various metrics. Dashed curves represent the RL baseline and (no-CoT) SFT; green shades for the logprob family of rewards (Logprob, Average logprob and JEPO) and blue for probability-based rewards {\em Probability (VeriFree)} and {\em Average Probability} (RLPR). Numerical values can be found in \Cref{tab:verifiable_g32}.}
    \label{fig:llama_math_g32}
\end{figure}

Sampling answers at temperature $T=1$ generally makes performance worse
across the board. It also affects the ranking of methods: 
methods that use logprobs or average logprobs underperform both Base RL and the Prob variant.
We believe $T=1$ sampling is the reason why logprobs did not perform well
on MATH in \cite{tang2025beyond}. 
Overall, we do not detect any strong difference between JEPO and simple Logprob when greedy sampling is used. Conceptually, JEPO is a
more precise, more computationally heavy version of Logprob (larger $N$
for Monte Carlo estimation of log-probabilities, see
Section~\ref{sec:method}). The additional complexity is not justified in
our setting.

The picture shifts when we consider \emph{perplexity} in addition to success
rate: here, only Logprob,
AvgLogprob and JEPO achieve good perplexities, improving SFT by a
significant margin on this criterion.  This is new evidence of the
interest of a CoT for these domains.
Perplexity may not be the metric of
most direct interest for verifiable questions, but it nevertheless
informs us on the qualitative behavior of different models. Base RL and
Prob yield very poor perplexities: a prob-trained model makes little difference between predicting
a wrong answer with probability $0.99$ or $1$ and giving the correct
answer with probability $0.01$ or $0$, while this makes a large difference for
log-probabilities.
On the other hand, logprob-trained models make sure
that if they are wrong, they are not \textit{confidently} wrong, by
attributing some nonzero probability to all plausible answers.

Overall, logprob-trained models get both good success rates and
good perplexity, while models trained directly for the success rate
sacrifice perplexity.  Presumably, logprob-trained models smooth out
their predictions, while verifier or probability-based variants emit
``sharper'' probabilities.

\subsection{Results on Non-verifiable Domains}

\input{tables/nonverifiable_nonreg}

\begin{figure}
    \centering
    \includegraphics[width=\linewidth]{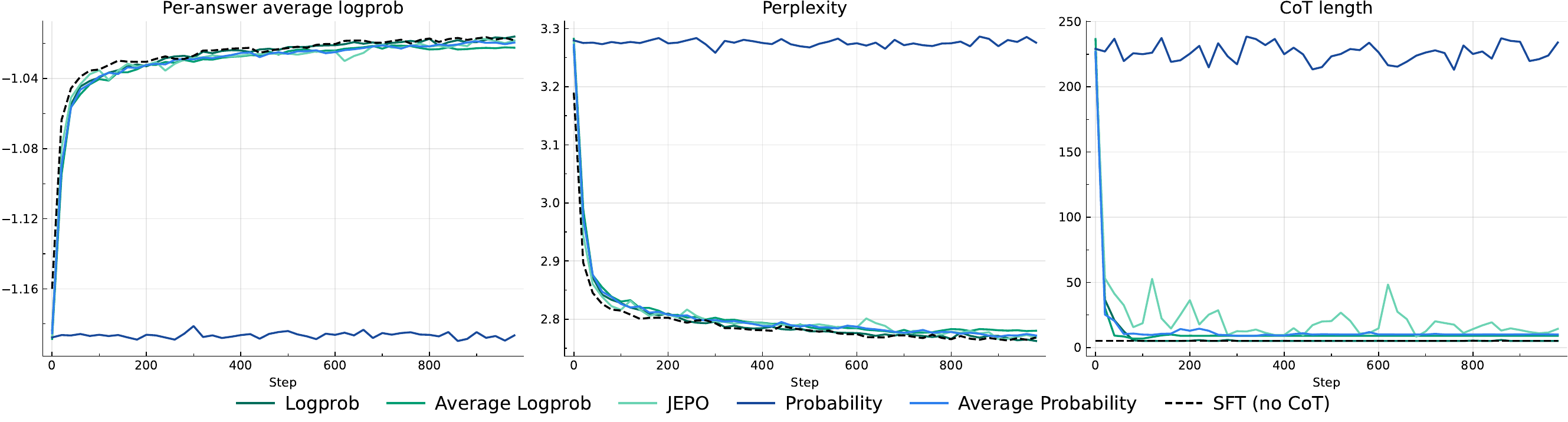}
    \caption{\small {\bf Non-verifiable: Qwen 2.5 3B Instruct on NuminaProof.} Learning curves of our algorithms for three metrics. Numerical values can be found in \Cref{tab:nonverifiable_base}. Log-prob family models match the (per-answer) average log-prob and perplexity from SFT, while probability rewards fail to improve on these metrics due to the sparsity of the rewards. We observe a rapid ``collapse" in CoT-length for the log-prob family. }
    \label{fig:qwen_numina_base}
    \vspace{-12pt}
\end{figure}

We present the results on NuminaProof and Alpaca with the Llama and Qwen models in Table~\ref{tab:nonverifiable_base} and \Cref{fig:qwen_numina_base} and \Cref{fig-app:llama_numina_base,fig-app:llama_alpaca_base,fig-app:qwen_alpaca_base} in \Cref{app:experiments}. We observe that training with logprobs, or with average logprobs or JEPO, consistently matches the performance of SFT. As predicted, {\em Probability} (VeriFree) fails to improve on these metrics, and {\em Average Probability} (RLPR) is noisier but trails the logprob family closely. This establishes the log-prob family of rewards as a universal method for both verifiable and non-verifiable domains.

\subsection{Length of the Chain-of-Though During Training}
\label{sec:length}

We now report some intriguing observations on the behavior of the CoT during training, for which we have no complete explanation. 
In Figure~\ref{fig:llama_math_g32}, we see that CoTs trained with Logprob variants show an
initial dip in length, followed by a recovery in verifiable domains. This
pattern does not occur for Prob variants or Base RL.
For non-verifiable rewards, we see an even starker pattern in Figure~\ref{fig:qwen_numina_base}: the CoT dips to a length of $~10$ tokens (including formatting tokens) and never recovers. This means that the CoT is largely eliminated, and Logprob methods effectively become SFT -- indeed we observe that the perplexity of methods with a collapsed CoT closely match those of the SFT baseline.

To understand the mechanism behind the CoT length dip, we hypothesized that early in training, shorter CoTs might lead to better predictions since the base model has been trained without CoTs. An initial negative correlation between CoT length and reward may push the model towards shorter CoTs during reinforcement learning.
Indeed, we find that on average over questions in a dataset, the initial
model exhibits a negative correlation between CoT length and
log-probability of the correct answer for a given question
(Appendix~\ref{sec:correlation}, Figs.~\ref{fig-app:numina_corr_color} and
\ref{fig-app:math_corr_color}), both for Math and NuminaProof. This
explains the initial drive towards shorter CoTs when doing RL on logprob
rewards. This correlation is not present with probability rewards
(Fig.~\ref{fig-app:mathprob_corr_color}), as the signal is visibly
``squashed'' for small probabilities.

We tried two types of interventions on this pattern: increasing the KL divergence regularization to the base model, and introducing a length penalty that adds a negative reward for every token below a certain threshold in the CoT (see \Cref{app:KL_length} for details). These interventions worked in that they prevented the CoT length dip, but this came at the cost of actual performance, as shown in \Cref{fig-app:ablations_l3b_numina,fig-app:ablations_q3b_numina,fig-app:ablations_l3b_alpaca,fig-app:ablations_q3b_alpaca} in \Cref{app:KL_length}.  

We also considered that the SFT term might overpower the RL part, as the
model is not used to writing long proofs; or that, initially, the mere presence of a CoT
might affect the quality of the answer negatively (eg, due to distance
from the question). 
To address these two points, we produced a ``warm-start'' model by SFT-ing
on the answers in the presence of CoTs (with masked gradients for the CoT
part). Such a warm-start model can be used in turn for RL fine-tuning.
We describe these results in \Cref{app:warmstart}: this stabilizes CoT
length, but final performance only matches the SFT baseline, without beating
it under a reasonable compute budget.

It is worth noting that in a similar setting, \citet{tang2025beyond} show JEPO eventually exceeding the SFT performance with long-form answers. The critical difference is that \citet{tang2025beyond} train for significantly longer (an order of magnitude) at a larger batch size and lower learning rate.  So it is possible that JEPO enables training with long-form answers, at the cost of much higher compute requirements compared to RL with short-form, verifiable answers.

\paragraph{Discussion: Why don't CoTs improve performance for non-verifiable domains?} 
We can put forward several hypothetical explanations for the apparent
lack of improvement from CoT in non-verifiable domains.
One possibility is that it takes longer to train long CoTs than short CoTs. RL has a worse signal-to-noise ratio when the number of actions increases, because of the well-known ``credit-assignment problem''. 
For CoT training, the actions are the tokens, so it is harder to identify good correlations in long CoTs than short ones. 
If this is the case, then a correlation between short CoTs and better
performance would be present in the early phases of training, only to disappear once long CoTs catch up in performance. 
This could explain the dip-and-recovery pattern for verifiable domains.
However, this does not explain why this pattern occurs for Log-prob but not for Prob in verifiable domains. Also, in this situation, we would expect the length penalties %
to help.

Another possibility is that, with long answers, the model has time to deploy a hidden CoT within its internal layers. 
Indeed, the survey by \citet{zhu2025surveylatentreasoning}  puts forward increasing evidence for the existence of such hidden CoTs in LLMs. Conversely, for the very short answers in verifiable domains, there may be too few tokens to have an efficient internal CoT during the answer, and an actual, non-hidden CoT may be necessary. 
If this is the case, it may be interesting to build datasets that interpolate between short and long answers, and see if there is a transition. We leave these investigations to future work.

%% file: tables/verifiable_g32.tex
\begin{table*}[t]
\centering
\footnotesize
\setlength{\tabcolsep}{5pt}
\renewcommand{\arraystretch}{1.15}
\resizebox{\textwidth}{!}{%
\begin{tabular}{lccccccc}
\toprule
 & Base model & Base RL & Log-prob & Avg Logprob & Probability & Avg Probability & SFT (no CoT) \\
\midrule
\addlinespace\textbf{Llama 3B, MATH} &  &  &  &  &  &  &  \\
Greedy success & 17.13 ± 0.00 & 42.74 ± 0.66 & 43.30 ± 0.10 & 43.66 ± 0.25 & 44.19 ± 0.57 & 43.95 ± 0.44 & 14.43 \\
T=1 Sampled Success & 10.77 ± 0.27 & 41.36 ± 0.15 & 34.66 ± 0.81 & 34.83 ± 0.00 & 41.39 ± 0.09 & 38.38 ± 0.74 & 10.00 \\
Average log-prob MC32 & — & — & -0.68 ± 0.00 & -0.69 ± 0.01 & -1.32 ± 0.01 & -0.77 ± 0.00 & -0.97 \\
Average log-prob & -4.21 ± 0.00 & -2.63 ± 0.02 & -0.79 ± 0.03 & -0.81 ± 0.04 & -1.97 ± 0.04 & -1.08 ± 0.03 & -0.97 \\
Perplexity & 67.29 ± 0.00 & 13.87 ± 0.34 & 2.21 ± 0.06 & 2.25 ± 0.09 & 7.14 ± 0.26 & 2.95 ± 0.08 & 2.63 \\
CoT length & 326.77 ± 0.71 & 321.01 ± 24.42 & 298.97 ± 2.82 & 302.74 ± 2.21 & 313.58 ± 1.87 & 320.17 ± 5.69 & 5.00 \\
\addlinespace\textbf{Qwen 3B, MATH} &  &  &  &  &  &  &  \\
Greedy success & 21.19 ± 0.00 & 55.85 ± 0.46 & 56.84 ± 0.21 & 56.11 ± 0.20 & 56.46 ± 0.05 & 57.41 ± 0.16 & 18.32 \\
T=1 Sampled Success & 16.63 ± 0.30 & 55.36 ± 0.17 & 44.18 ± 0.42 & 42.45 ± 0.51 & 53.32 ± 0.22 & 47.34 ± 0.87 & 12.00 \\
Average log-prob MC32 & — & — & -0.39 ± 0.00 & -0.40 ± 0.01 & -1.03 ± 0.01 & -0.42 ± 0.00 & — \\
Average log-prob & -2.19 ± 0.00 & -2.11 ± 0.10 & -0.44 ± 0.02 & -0.50 ± 0.00 & -1.77 ± 0.01 & -0.68 ± 0.05 & -0.69 \\
Perplexity & 8.92 ± 0.00 & 8.25 ± 0.80 & 1.55 ± 0.03 & 1.64 ± 0.00 & 5.85 ± 0.03 & 1.97 ± 0.09 & 1.99 \\
CoT length & 222.35 ± 0.73 & 451.83 ± 14.14 & 381.49 ± 10.22 & 372.01 ± 18.76 & 429.35 ± 1.25 & 419.51 ± 4.31 & 5.00 \\
\addlinespace\textbf{Llama 3B, DeepScaleR} &  &  &  &  &  &  &  \\
Greedy success & 10.05 ± 0.00 & 28.55 ± 0.14 & 32.40 ± 1.02 & 32.20 ± 1.37 & 31.75 ± 1.56 & 30.75 ± 1.02 & 9.88 \\
T=1 Sampled Success & 5.20 ± 0.14 & 28.54 ± 0.35 & 17.97 ± 0.98 & 18.32 ± 0.04 & 28.47 ± 0.03 & 21.71 ± 0.50 & 5.14 \\
Average log-prob MC32 & — & -2.01 ± 0.29 & -0.79 ± 0.01 & -0.80 ± 0.00 & -1.63 ± 0.03 & -0.85 ± 0.01 & -1.01 \\
Average log-prob & -4.92 ± 0.00 & -2.77 ± 0.48 & -0.83 ± 0.04 & -0.84 ± 0.02 & -2.32 ± 0.15 & -1.31 ± 0.08 & -1.01 \\
Perplexity & 137.25 ± 0.00 & 16.89 ± 7.75 & 2.30 ± 0.09 & 2.32 ± 0.05 & 10.28 ± 1.49 & 3.71 ± 0.30 & 2.74 \\
CoT length & 354.91 ± 1.28 & 294.52 ± 22.73 & 437.11 ± 0.41 & 371.88 ± 11.88 & 340.01 ± 6.16 & 335.89 ± 19.74 & 5.00 \\
\addlinespace\textbf{Qwen 3B, DeepScaleR} &  &  &  &  &  &  &  \\
Greedy success & 14.02 ± 0.00 & 38.30 ± 2.40 & 37.92 ± 0.11 & 38.12 ± 0.83 & 38.68 ± 0.67 & 39.20 ± 1.56 & 11.52 \\
T=1 Sampled Success & 11.00 ± 0.42 & 38.56 ± 0.35 & 25.79 ± 0.43 & 25.82 ± 0.44 & 36.83 ± 0.07 & 28.79 ± 0.59 & 6.48 \\
Average log-prob MC32 & — & -2.38 ± 0.00 & -0.56 ± 0.00 & -0.57 ± 0.01 & -1.68 ± 0.05 & -0.59 ± 0.00 & -0.76 \\
Average log-prob & -2.57 ± 0.00 & -3.55 ± 0.14 & -0.59 ± 0.01 & -0.70 ± 0.01 & -2.73 ± 0.03 & -0.85 ± 0.02 & -0.76 \\
Perplexity & 13.12 ± 0.00 & 34.85 ± 4.93 & 1.80 ± 0.02 & 2.02 ± 0.01 & 15.28 ± 0.42 & 2.33 ± 0.05 & 2.14 \\
CoT length & 218.12 ± 7.35 & 453.69 ± 7.81 & 532.79 ± 36.46 & 463.46 ± 0.11 & 481.15 ± 3.03 & 460.00 ± 6.25 & 5.00 \\
\bottomrule
\end{tabular}

}
\caption{\small {\bf Results on verifiable domains, G=32.} 
Final performance of models  across all our algorithms and metrics.
Results are averaged over two seeds. Rows are labeled by the test
metrics, columns by the algorithms. We observe that methods which use the
log-probability as a reward (Log-prob, Avg Logprop, JEPO) often
underperform the baseline when the answer is sampled. However, the gap
closes when the answer is produced deterministically ({\em greedy
success}). Perplexity and log-prob based metrics universally improve for
the log-prob family of rewards, clearly surpassing SFT levels, while base RL lags behind in this metric, and probability-based rewards situate themselves in the middle between those.
Learning curves are shown in \Cref{fig:llama_math_g32,fig-app:qwen_math_g32,fig-app:llama_ds_g32,fig-app:qwen_ds_g32}. }
\label{tab:verifiable_g32}
\vspace{-15pt}
\end{table*}

%% file: tables/nonverifiable_nonreg.tex
\begin{table*}[t]
\centering
\scriptsize
\setlength{\tabcolsep}{5pt}
\renewcommand{\arraystretch}{1.15}
\resizebox{\textwidth}{!}{%
\begin{tabular}{lccccccc}
\toprule
 & Base model & Log-prob & Avg Logprob & Probability & Avg Probability & JEPO & SFT (no CoT) \\
\midrule
\addlinespace\textbf{Llama 3B, NuminaProof} &  &  &  &  &  &  &  \\
Per-answer avg logprob & -1.2871 & -1.1124 & -1.1175 & -1.2859 & -1.1157 & \textbf{-1.1104} & -1.1127 \\
Perplexity & 3.62 & \textbf{3.04} & 3.06 & 3.62 & 3.05 & \textbf{3.04} & \textbf{3.04} \\
Per-answer avg logprob MC32 & — & -1.1124 & -1.1175 & -1.2850 & -1.1157 & \textbf{-1.1102} & -1.1127 \\
CoT length & 474.0 & 9.1 & 9.0 & 469.6 & 9.1 & 34.1 & 5.0 \\
\addlinespace\textbf{Qwen 3B, NuminaProof} &  &  &  &  &  &  &  \\
Per-answer avg logprob & -1.1838 & -1.0174 & -1.0235 & -1.1862 & -1.0217 & -1.0218 & \textbf{-1.0172} \\
Perplexity & 3.27 & \textbf{2.77} & 2.78 & 3.27 & 2.78 & 2.78 & \textbf{2.77} \\
Per-answer avg logprob MC32 & — & -1.0174 & -1.0235 & -1.1852 & -1.0217 & -1.0218 & \textbf{-1.0172} \\
CoT length & 225.6 & 5.3 & 9.0 & 225.2 & 10.0 & 14.1 & 5.0 \\
\addlinespace\textbf{Llama 3B, Alpaca} &  &  &  &  &  &  &  \\
Per-answer avg logprob & -1.3493 & -0.9397 & -0.9449 & -1.5772 & -0.9513 & -0.9443 & \textbf{-0.9381} \\
Perplexity & 3.85 & \textbf{2.56} & 2.57 & 4.84 & 2.59 & 2.57 & \textbf{2.56} \\
Per-answer avg logprob MC32 & — & -0.9396 & -0.9449 & -1.5724 & -0.9512 & -0.9436 & \textbf{-0.9381} \\
CoT length & 134.2 & 14.2 & 14.1 & 58.6 & 9.1 & 14.8 & 5.0 \\
\addlinespace\textbf{Qwen 3B, Alpaca} &  &  &  &  &  &  &  \\
Per-answer avg logprob & -1.3968 & \textbf{-0.8903} & -0.8933 & -1.2982 & -0.8989 & -0.8976 & -0.8905 \\
Perplexity & 4.04 & \textbf{2.44} & \textbf{2.44} & 3.66 & 2.46 & 2.45 & \textbf{2.44} \\
Per-answer avg logprob MC32 & — & \textbf{-0.8902} & -0.8931 & -1.2955 & -0.8988 & -0.8969 & -0.8905 \\
CoT length & 83.7 & 16.7 & 14.3 & 16.6 & 15.2 & 16.6 & 5.0 \\
\bottomrule
\end{tabular}

}
\caption{\small {\bf Results on non-verifiable domains.} Final performance across all initial models and metrics, on non-verifiable datasets. Probability rewards fail to learn due to their extremely low rewards. We observe that methods which use the log-probability experience a CoT collapse, reducing to SFT.  The corresponding learning curves are shown in \Cref{fig:qwen_numina_base} and \Cref{fig-app:llama_numina_base,fig-app:llama_alpaca_base,fig-app:qwen_alpaca_base} in \Cref{app:experiments}.}
\label{tab:nonverifiable_base}
\vspace{-12pt}
\end{table*}

%% file: conclusion.tex
\section{Conclusion}%

Our work establishes log-probability rewards as a unifying training
signal effective in both verifiable and non-verifiable domains, without
relying on ground-truth correctness labels. On reasoning benchmarks like
MATH and DeepScaleR, log-probability rewards match the success rates of
standard 0/1 RL objectives while substantially improving perplexity; on
long-form proofs, they match supervised fine-tuning while other
probability-based variants fall well below. This shows that the
same criterion can be carried seamlessly across settings. This highlights
their potential as a general recipe for post-training reasoning LLMs,
valid over the full range of possible answer types. In future work, we hope to further develop this approach on non-verifiable domains, enabling efficient RL training on any dataset, leveraging the CoT capabilities.

\paragraph{\em Acknowledgement.} JK thanks the Simons Foundation for support through the Collaborative Grant ``The Physics of Learning and Neural Computation''.

%% file: proofs.tex
\section{Losses and Advantages for the Rewards Considered}\label{app:losses}

\begin{lemma}
Let $z$ be a chain-of-thought variable sampled from a model $\pi_\theta$
with parameters $\theta$, and let $R_\theta(z)$ be a reward function that
depends on $z$ and also possibly on $\pi_\theta$ (for instance,
$R_\theta(z)=\log \pi_\theta(a^\star|z)$ or
$R_\theta(z)=\pi_\theta(a^\star|z)$).

Then the expected
reward
\begin{equation}
J_\theta=\E_{z\sim \pi_\theta} [R_\theta(z)]
\end{equation}
has the same gradients (up to sign) as the loss function
\begin{equation}
L_\theta=\E_{z\sim \pi_\theta\sg} \left[
(R_\theta(z) -c_\theta)\sg \log \pi_\theta(z)+R_\theta(z)
\right]
\end{equation}
where ${}\sg$ denotes a stop-grad operator, and $c_\theta$ is any
expression
independent of $z$.
\end{lemma}

For instance, in RLOO, $c_\theta$ is the average of $R_\theta(z')$
over samples $z'\sim \pi_\theta$ independent from $z$.

\begin{proof}
The gradient of $J_\theta$ is
\begin{align}
\nabla_\theta \E_{z\sim
\pi_\theta}[R_\theta(z)]
&=\nabla_\theta\sum_z \pi_\theta(z)R_\theta(z)
\\&=\sum_z (\pi_\theta(z) \nabla_\theta \log
\pi_\theta(z))R_\theta(z)+\sum_z \pi_\theta(z) \nabla_\theta
R_\theta(z)
\\&=\E_{z\sim \pi_\theta}\left[
R_\theta(z)\nabla_\theta \log
\pi_\theta(z))+\nabla_\theta R_\theta(z)
\right]
\end{align}
hence the statement without $c_\theta$.

Now we have $\E_{z\sim p_\theta} \nabla_\theta \log \pi_\theta(z)=\sum_z
\pi_\theta(z) \nabla_\theta\log \pi_\theta(z)=\sum_z \nabla_\theta
\pi_\theta(z)=\nabla_\theta 1=0$. Therefore, we have $\E_{z\sim
\pi_\theta}[c_\theta \nabla_\theta \log \pi_\theta(z)]=0$ as long as
$c_\theta$ is independent of $z$. Hence we can subtract $c_\theta
\nabla_\theta \log \pi_\theta(z)$ from the expression above, which leads to the conclusion.
\end{proof}

\section{Experimental details}\label{app:expdetail}

For each experiment, we use a synchronous implementation of RLOO running in parallel across 8 processes. We use the AdamW~\citep{kingma2014adam} optimizer with a learning rate of $10^{-5}$, and a cosine schedule with a 20 step warm-up. 
During our research, we tried a few learning rates for the probability rewards, but noticed that the chosen value worked consistenly for all variants.
We clip the global gradient norm to a global threshold of $1.0$. Each batch contains $8$ questions from the dataset with $G$ different CoTs; such a batch corresponds to one {\em step} in all our figures. 

{\bf Full Details on the Datasets.}
We consider two \emph{verifiable} math benchmarks and two \emph{non-verifiable} long-form datasets.
(i) \textbf{MATH}~\citep{HendrycksBKABTS21}: we concatenate all official subsets, parse the final answer from \verb|\boxed{...}|, discard intermediate solutions, and hold out a random $10\%$ for validation. We report accuracy on the official test split. The resulting training set contains $\sim$7{,}000 short-answer problems.
(ii) \textbf{DeepScaleR (Preview)}~\citep{deepscaler2025}: we discard long solutions, use the provided final answer as ground truth, hold out a random $10\%$ for validation, and report performance on this held-out set. The training set has $\sim$39{,}000 short-answer problems.
(iii) \textbf{Alpaca (cleaned)}~\citep{alpaca}: we use the standard cleaned variant; $1{,}000$ random examples are used for validation, leaving $\sim$50{,}000 training samples with predominantly long-form answers.
(iv) \textbf{NuminaProof}: starting from \textsc{NuminaMath-1.5}~\citep{numina_math_datasets}, we filter for theorem–proof style items (full solutions are proofs), remove instances with hyperlinks, and sanitize the remaining solutions. We reserve $1{,}000$ examples for validation, yielding $\sim$50{,}000 long-form training samples.

\paragraph{Prompting and formatting.}
All experiments use a \textsc{DeepSeek-R1}–style instruction format~\citep{guo2025deepseek} with the instruction as the system prompt and the question as the user message, rendered with each model family’s standard instruct template (Llama or Qwen). We prefill the assistant turn with \texttt{"<think>"} to initiate the reasoning trace. The final sentence of the system prompt—encouraging concise, easily parsable answers—is enabled only in verifiable settings (see \autoref{tem:system}).

At each training step, each process receives a question prompt, and generates $G$ completions with a maximum length of $T$ tokens to that question. Unless noted otherwise, we use $G=32$ in verifiable domains, and $G=4$ in nonverifiable domains, and $T=1024$. We generate completions until they reach the pattern \texttt{</answer}, but for the likelihood-based rewards, we truncate the CoT at \texttt{<answer}. This is inspired by~\citet{zhou2025verifree} who pointed out that in both the Llama and Qwen tokenizers, there is no individual token that contains the pattern \texttt{r>}, and thus it is guaranteed to be a consistent token boundary. 

For Base RL, 
the verifier tries to parse \verb|<answer>answer</answer>| and match it with the ground truth. If the answer is correct, the reward is 100. If the answer is incorrect but the format is kept correctly (parsing was succesful), the reward is 10. If the format is incorrect and an answer cannot be parsed, the reward is 0.
We train and evaluate with {\em exact} match on the answer.

\highlightbox{
\begin{template}[System prompt]
\label{tem:system}
\texttt{A conversation between User and Assistant. The user asks a question, and the Assistant solves it. The assistant first thinks about the reasoning process in the mind and then provides the user with the answer. The reasoning process and answer are enclosed within <think></think> and <answer></answer> tags, respectively, i.e., <think>reasoning process here</think> <answer>answer here</answer>. \textcolor{red}{Inside the answer tag, put only the answer and no additional commentary.} } 
\end{template}
}

\section{Additional Experimental Results}\label{app:experiments}

\subsection{Verifiable Domains}

Here, we complement \Cref{fig:llama_math_g32}   with the corresponding learning curves for other model-dataset combinations (\Cref{fig-app:qwen_math_g32,fig-app:llama_ds_g32,fig-app:qwen_ds_g32}) and provide the corresponding \Cref{fig-app:llama_math_g4,fig-app:qwen_math_g4,fig-app:llama_ds_g4,fig-app:qwen_ds_g4} and \Cref{tab:verifiable_g4} for training with $G=4$ (including JEPO, which for efficiency reasons we only ran for $G=4$).

\begin{figure}[h]
    \centering
    \includegraphics[width=\linewidth]{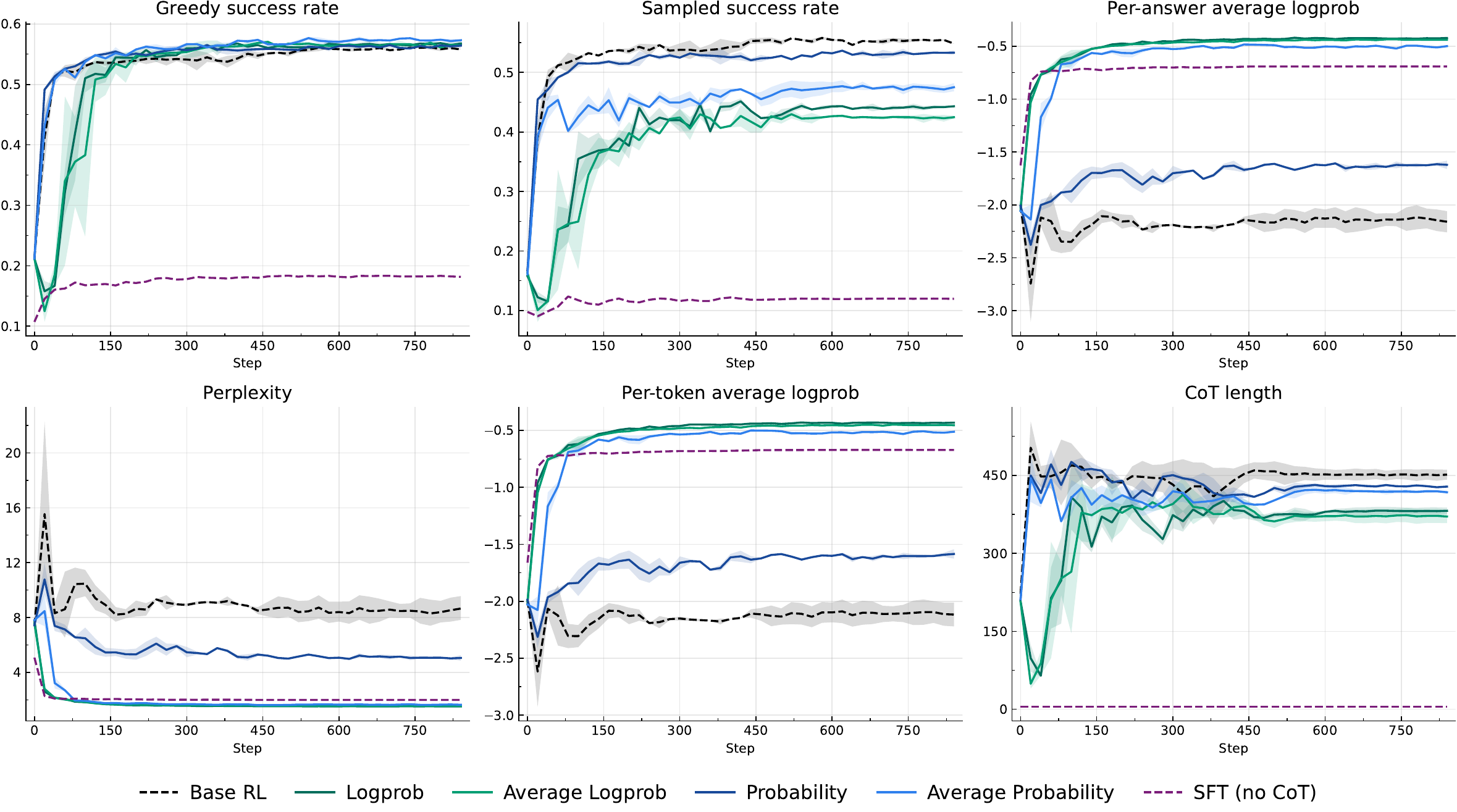}
    \caption{\small{\bf Verifiable. Qwen 2.5 3B Instruct on MATH with a group size of 32.} }
    \label{fig-app:qwen_math_g32}
\end{figure}

\begin{figure}[h]
    \centering
    \includegraphics[width=\linewidth]{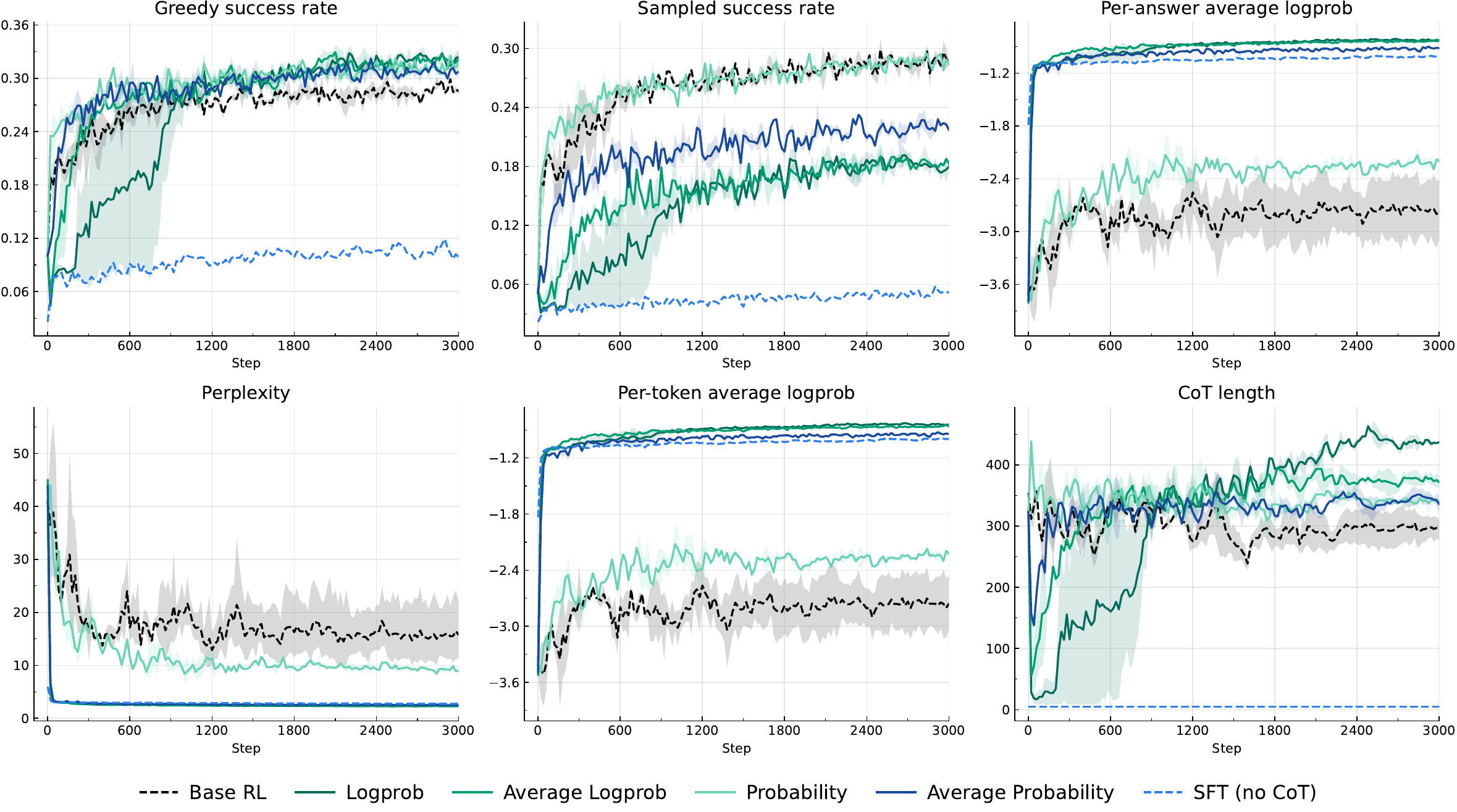}
    \caption{\small{\bf Verifiable. Llama 3.2 3B Instruct on DeepScaleR with a group size of 32.} }
    \label{fig-app:llama_ds_g32}
\end{figure}

\begin{figure}[h]
    \centering
    \includegraphics[width=\linewidth]{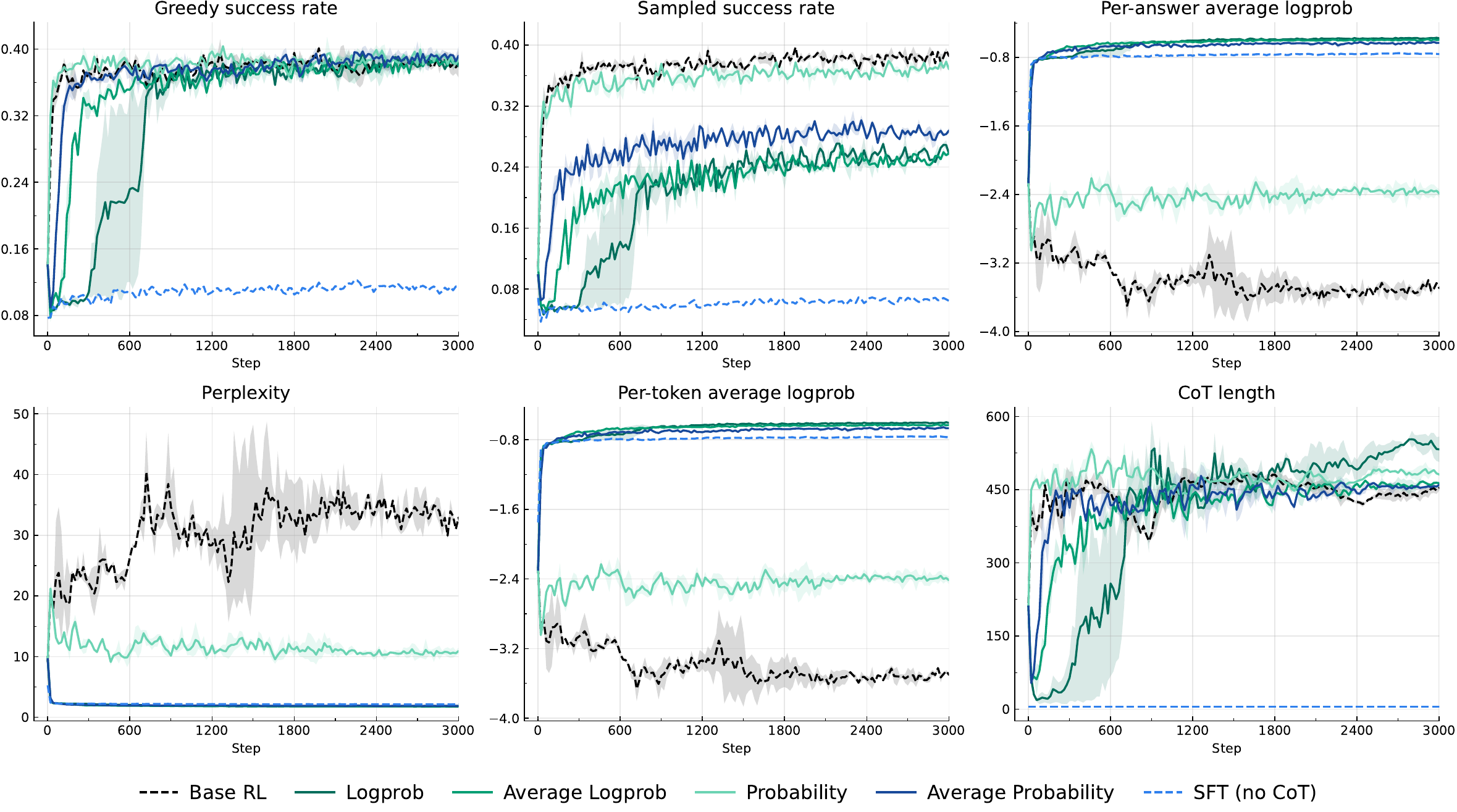}
    \caption{\small{\bf Verifiable. Qwen 2.5 3B Instruct on DeepScaleR with a group size of 32.} }
    \label{fig-app:qwen_ds_g32}
\end{figure}

\input{tables/verifiable_g4}

\begin{figure}[h]
    \centering
    \includegraphics[width=\linewidth]{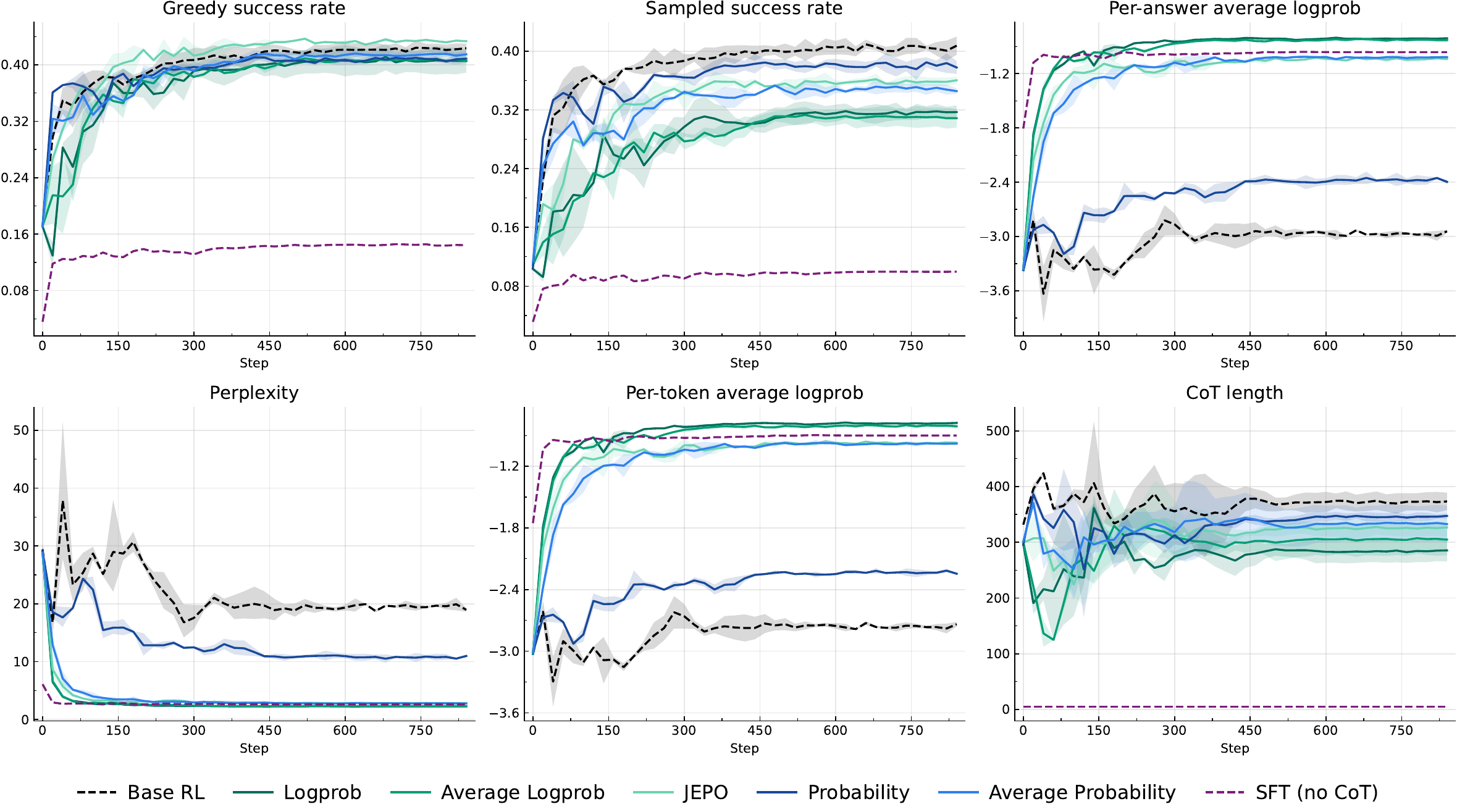}
    \caption{\small {\bf Verifiable: Llama-3.2-3B on MATH with a group size of 4.}  }
    \label{fig-app:llama_math_g4}
\end{figure}

\begin{figure}[h]
    \centering
    \includegraphics[width=\linewidth]{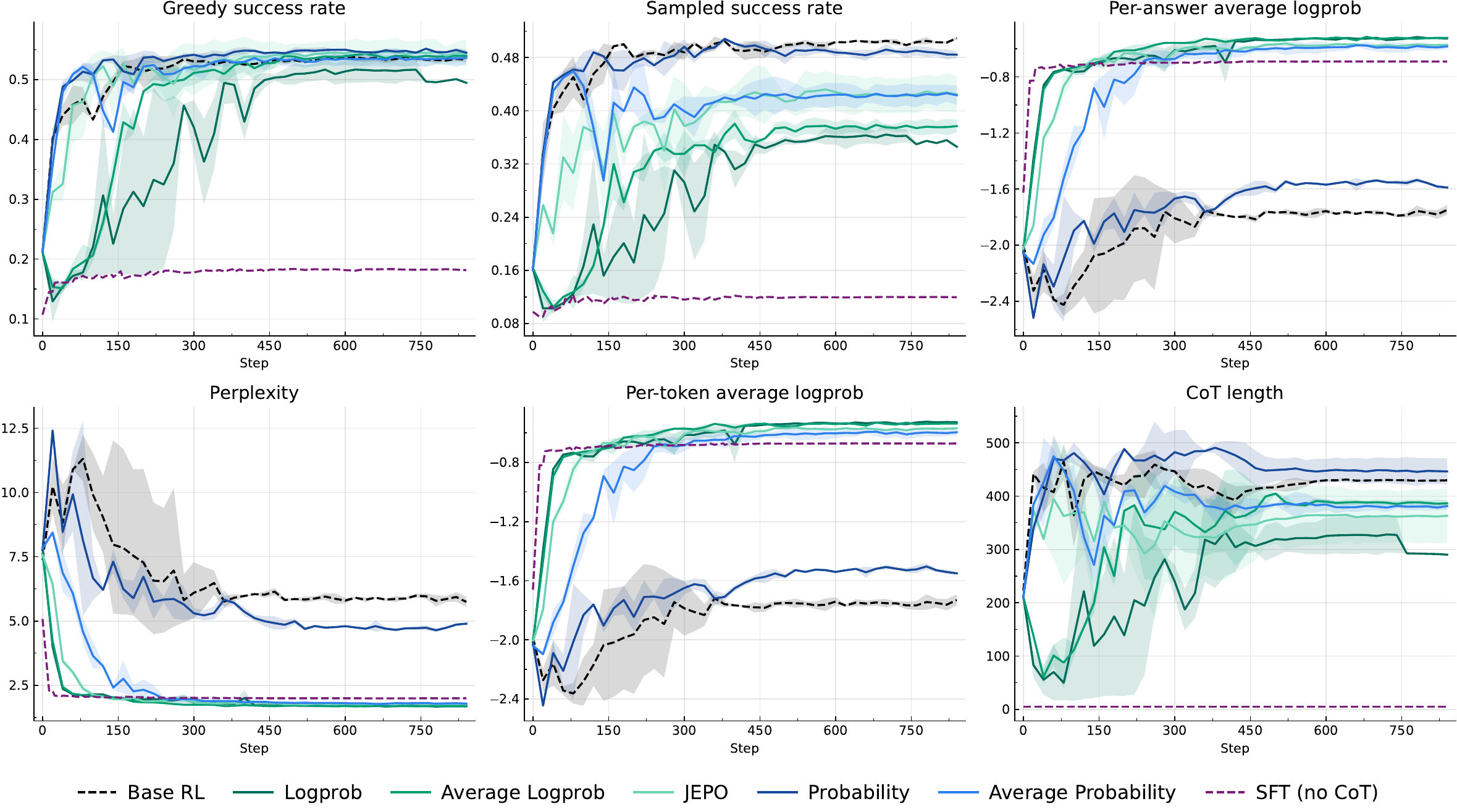}
    \caption{\small{\bf Verifiable. Qwen 2.5 3B Instruct on MATH with a group size of 4.}}
    \label{fig-app:qwen_math_g4}
\end{figure}

\begin{figure}[h]
    \centering
    \includegraphics[width=\linewidth]{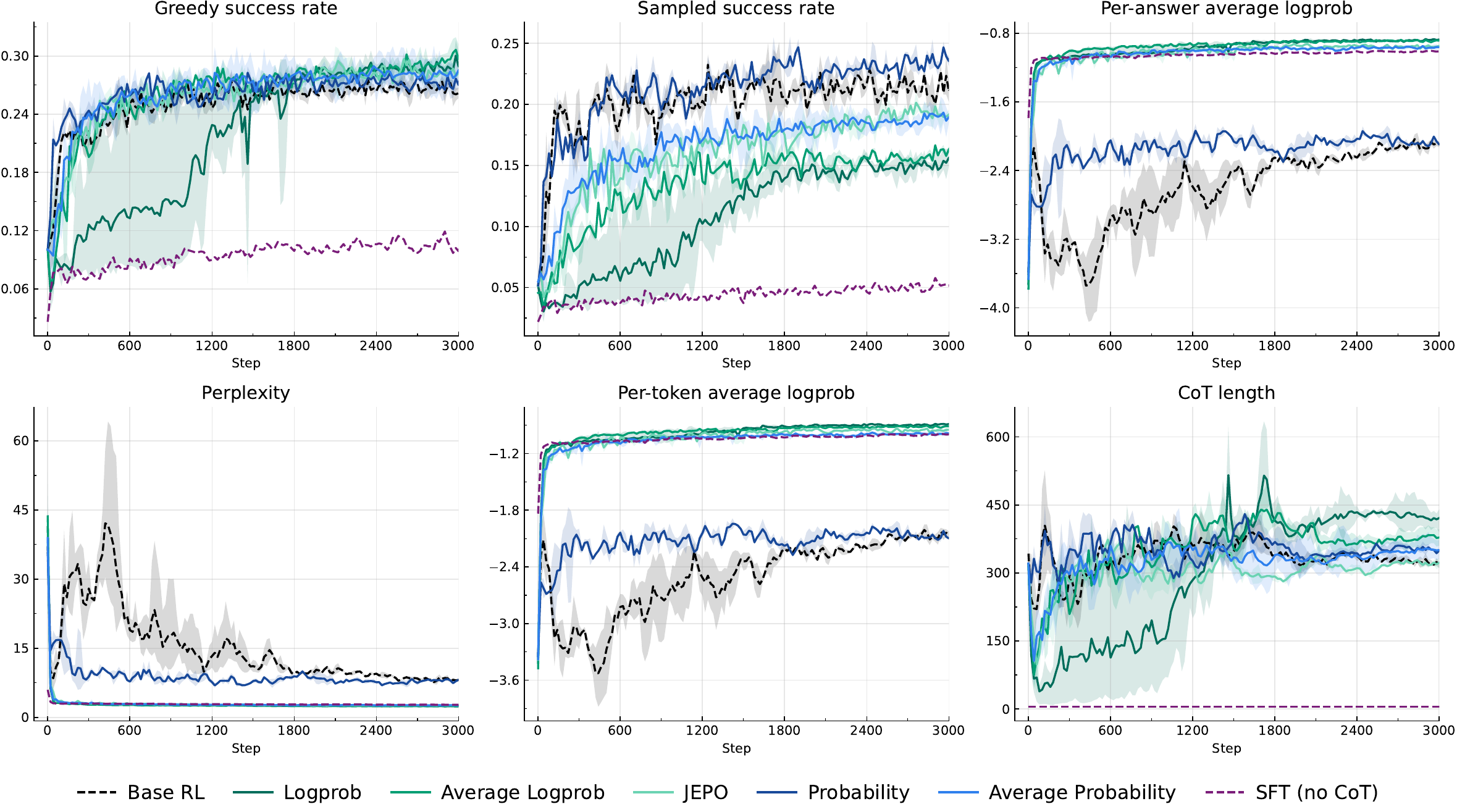}
    \caption{\small{\bf Verifiable. Llama 3.2 3B Instruct on DeepScaleR with a group size of 4.} }
    \label{fig-app:llama_ds_g4}
\end{figure}

\begin{figure}[h]
    \centering
    \includegraphics[width=\linewidth]{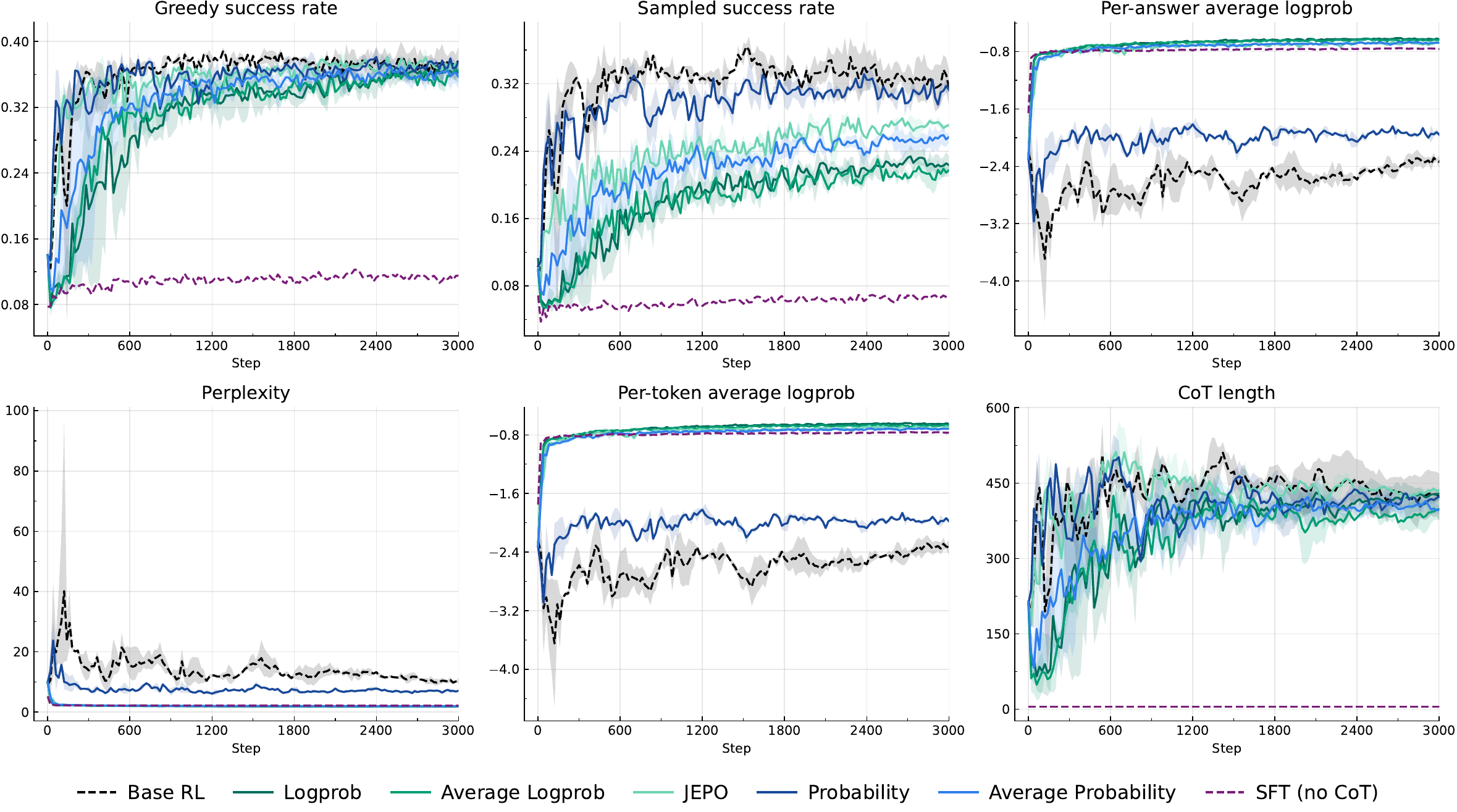}
    \caption{\small{\bf Verifiable. Qwen 2.5 3B Instruct on DeepScaleR with a group size of 4.} }
    \label{fig-app:qwen_ds_g4}
\end{figure}

\newpage 
\clearpage
\subsection{Non-verifiable Domains}

Here we provide the Figures complementary to \Cref{fig:qwen_numina_base} for other model/dataset combinations in \Cref{fig-app:llama_numina_base,fig-app:llama_alpaca_base,fig-app:qwen_alpaca_base}.

\begin{figure}[h]
    \centering
    \includegraphics[width=\linewidth]{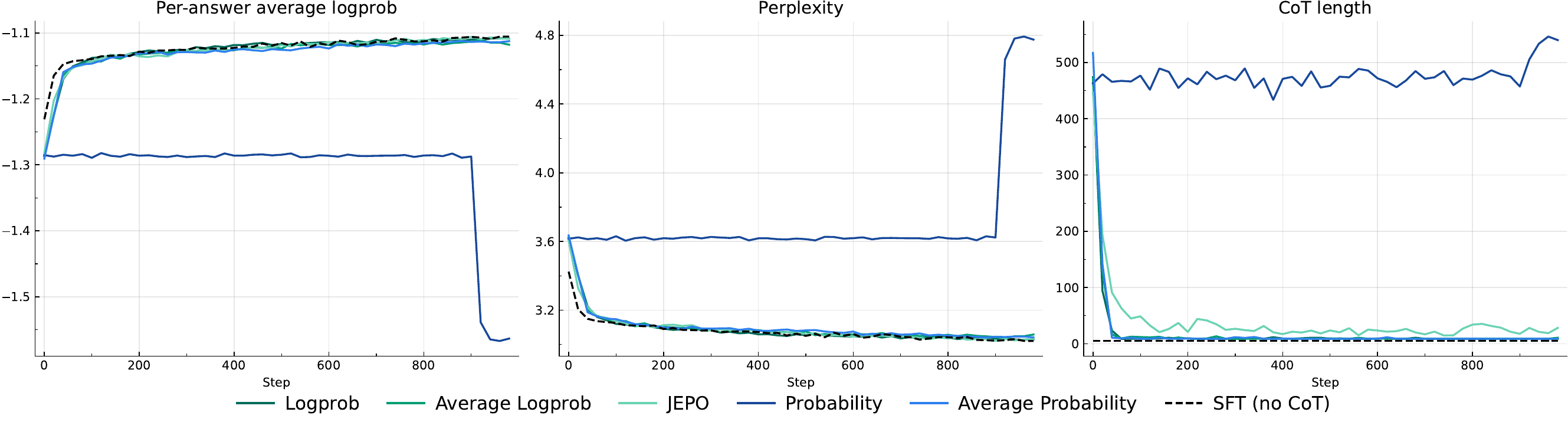}
    \caption{\small{\bf Non-verifiable. Llama 3.2 3B Instruct on NuminaProof.} }
    \label{fig-app:llama_numina_base}
\end{figure}

\begin{figure}[h]
    \centering
    \includegraphics[width=\linewidth]{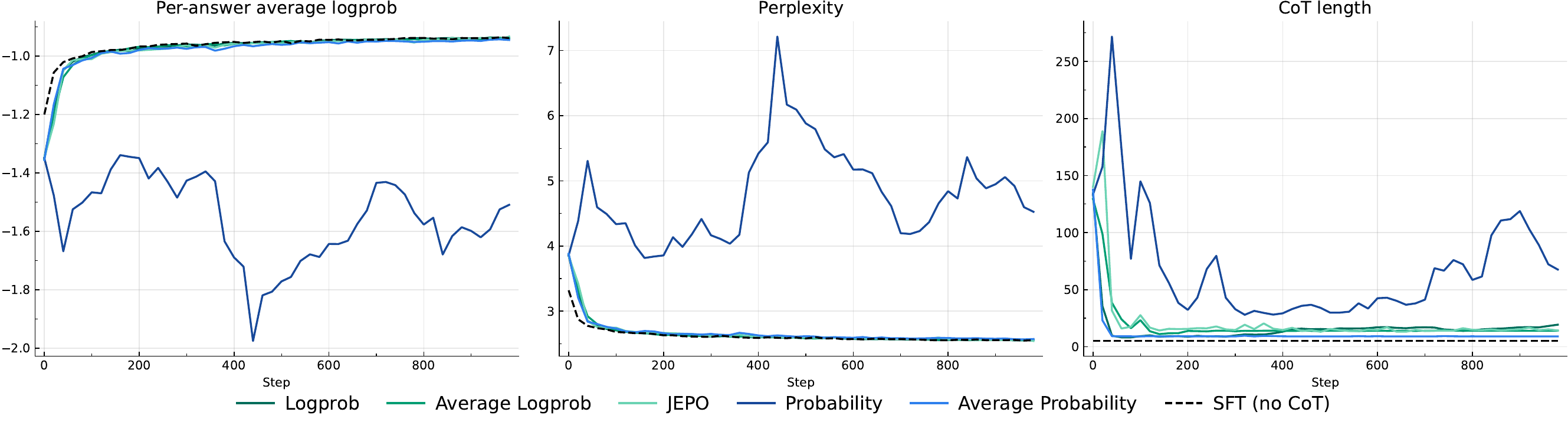}
    \caption{\small{\bf Non-verifiable. Llama 3.2 3B Instruct on Alpaca.} }
    \label{fig-app:llama_alpaca_base}
\end{figure}

\begin{figure}[h]
    \centering
    \includegraphics[width=\linewidth]{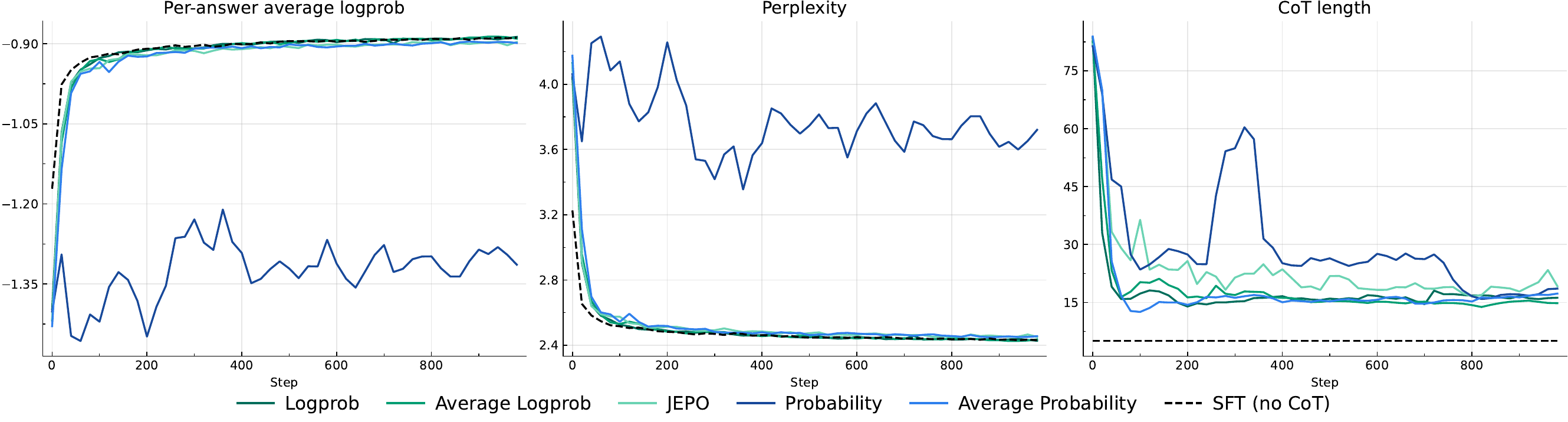}
    \caption{\small{\bf Non-verifiable. Qwen 2.5 3B Instruct on Alpaca.} }
    \label{fig-app:qwen_alpaca_base}
\end{figure}

\newpage
\section{Attempted regularization methods}\label{app:KL_length}

In Section~\ref{sec:length} we use two types of regularization to stabilize the CoT in nonverifiable domains. The first is straight-forward -- we include a KL divergence term in the loss, which keeps the model close to the initial model, as proposed by \citet{guo2025deepseek}.

The second type introduces an additional reward term:
\[
R_l(z) = r \cdot \min\{|z| - l_0, 0\}
\]
that is, for each missing token below a threshold for $l_0$, a negative reward $r$ is applied. We vary the threshold $l_0$ and report results for values of $100$, $150$, $300$ and $500$. To set the value of $r$, we design it so that it approximately compensates the increase of the reward during the initial CoT length drop over the initial 40 training steps. Specifically, we take the base nonverifiable experiments for each algorithm, model and dataset, and set 
\[
r = \frac{\Delta R}{\Delta L}
\]
where $\Delta R$ is the increase in validation reward, and $\Delta L$ is the decrease in the average CoT length.

The results of these ablations are presented in \Cref{fig-app:ablations_l3b_numina,fig-app:ablations_q3b_numina,fig-app:ablations_l3b_alpaca,fig-app:ablations_q3b_alpaca}.

\paragraph{Warm start.} 
\label{app:warmstart}

\begin{figure}[h]
    \centering
    \includegraphics[width=\linewidth]{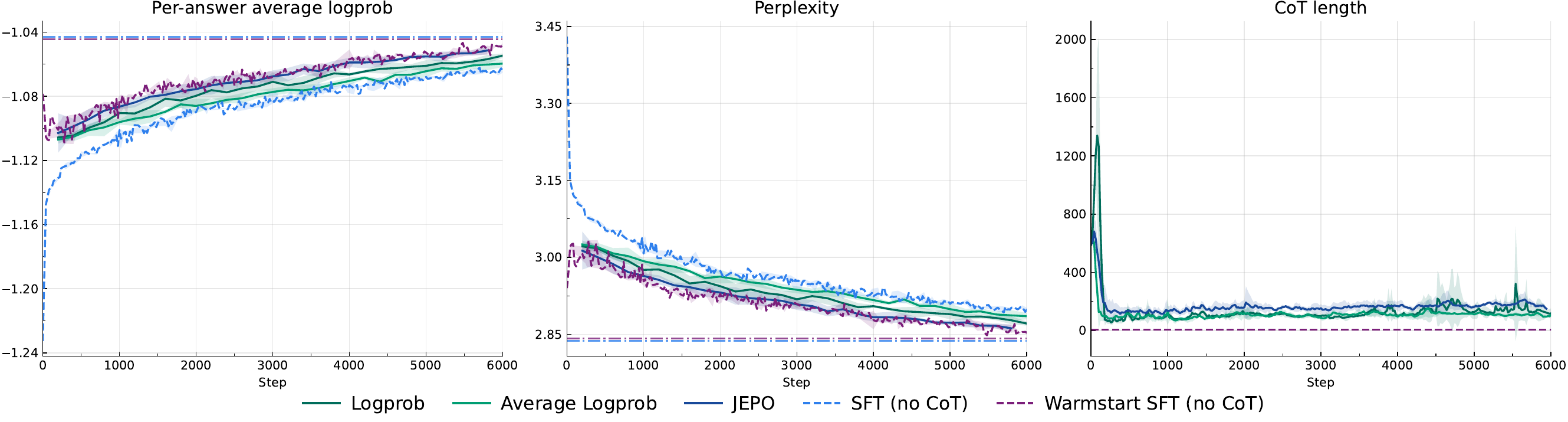}
    \caption{\textbf{Non-verifiable. Llama 3.2 3B Instruct-Warmstart on Numina.}}
    \label{fig:warmstart}
\end{figure}

When training the warm start models, our goal was to improve the initial
performance of the model. We observed that, at initialization, the
perplexity of the correct answer is better if it is appended directly
after the question, and worse if there is an autoregressively generated
CoT between them. This might indicate that the presence of a CoT is, by
itself, affecting performance negatively for the initial model.

Our next step was to produce a version of the base model that would be
``used to'' the presence of a CoT.  To this end, we generated a static
dataset of CoTs generated from the initial model, and trained it with SFT
on (question, completion, answer) triples, masking out everything but the
answer, with the intuition that this would train the model to produce good
answers in the presence of CoTs (but without training the CoT yet).

This ``warmstart'' model is then used as a starting point for the
various CoT training methods.

We display the results in \Cref{fig:warmstart} for Numina. This includes results obtained by training the warm-start model with logprob, average logprob, and JEPO rewards. There are also two SFT variants - one initialized with the typical checkpoint (Llama 3.2 3B-Instruct), and a warmstart variant, which is initialized from the same warmstart checkpoint as the RL models. The two dashed lines indicate the maximum performance that each SFT variant achieves after more training steps, but with equal compute to the RL curves.

We observe that warmstart initialization of the RL algorithms partly
stabilizes the CoT collapse. While the length of the CoT still drops to
significantly lower values, they tend to stabilize around 100-200 tokens,
instead of 5 tokens like in the coldstart case. This vindicates the
intuition behind warm-starting, namely, that initially the presence of a
CoT affects performance negatively.

However, even with warm-start and a stabilized CoT, the actual perplexity
stays close to the SFT baseline, and fails to beat it. It is possible that by adding significantly more compute, we could reproduce the findings of \citet{tang2025beyond} which show RL eventually beating SFT with JEPO.

\begin{figure}
    \centering
    \includegraphics[width=\linewidth]{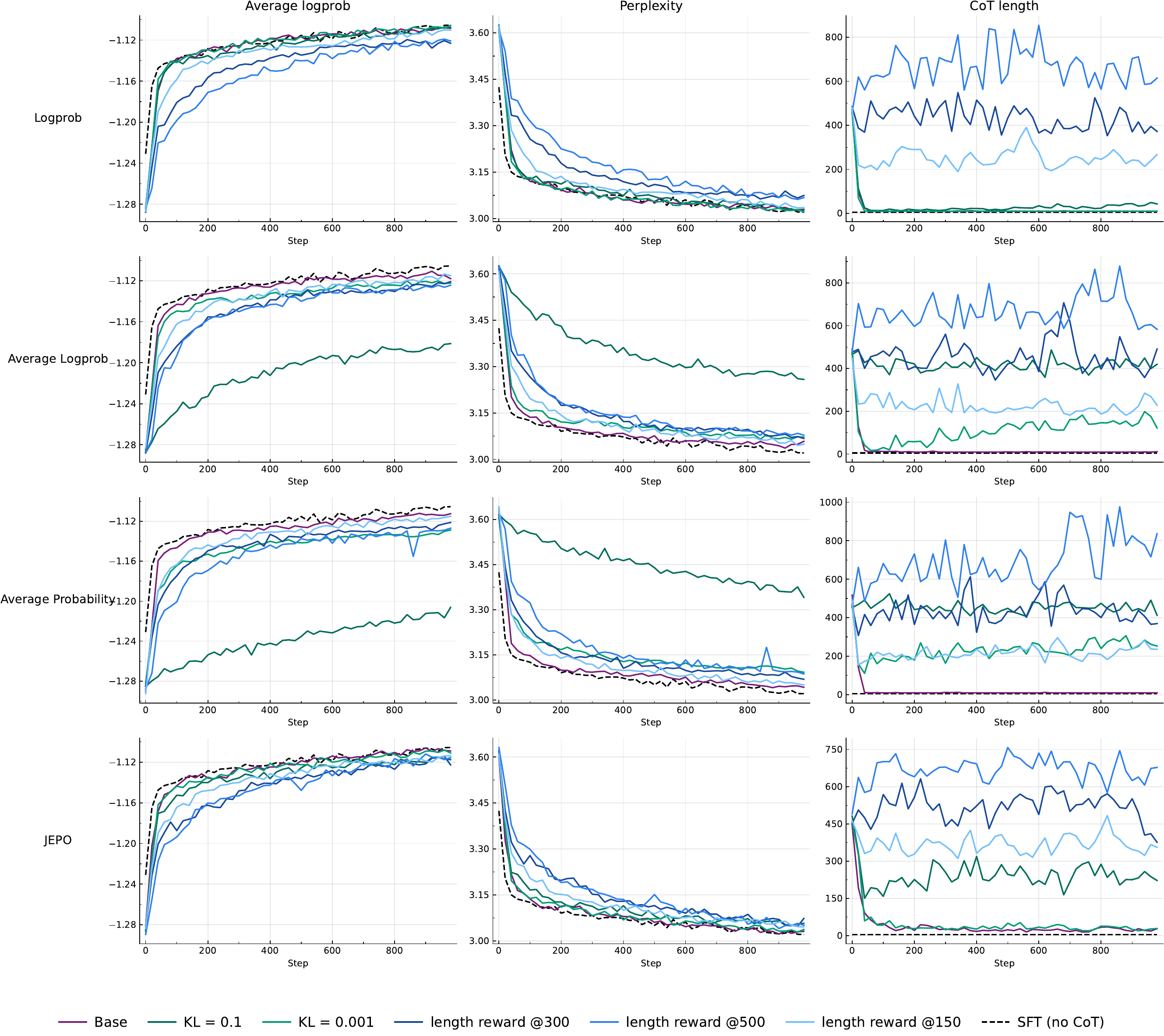}
    \caption{\small{\bf Non-verifiable. Llama 3.2 3B Instruct on NuminaProof.} Training curves of various attempts at stabilizing the CoT on nonverifiable domains with Llama 3.2 3B on NuminaProof. When the KL divergence coefficient, or the length threshold for the penalty are increased, the CoT does better at maintaining a non-trivial length. However, the actual log-prob of the correct answer decreases accordingly.}
    \label{fig-app:ablations_l3b_numina}
\end{figure}

\begin{figure}
    \centering
    \includegraphics[width=\linewidth]{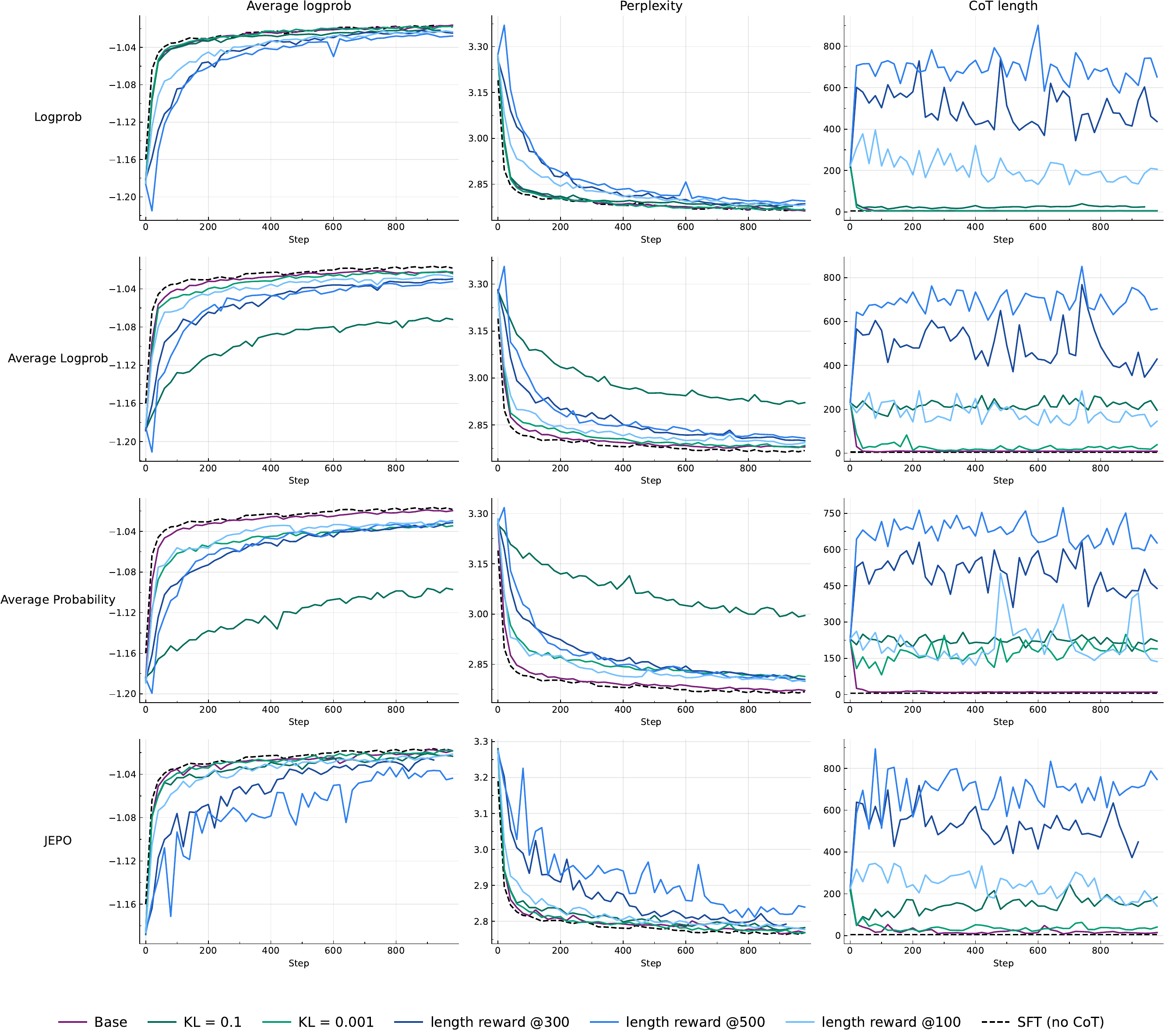}
    \caption{\small{\bf Non-verifiable. Qwen 2.5 3B Instruct on NuminaProof.} Conclusions are similar to \Cref{fig-app:ablations_l3b_numina}.}
    \label{fig-app:ablations_q3b_numina}
\end{figure}

\begin{figure}
    \centering
    \includegraphics[width=\linewidth]{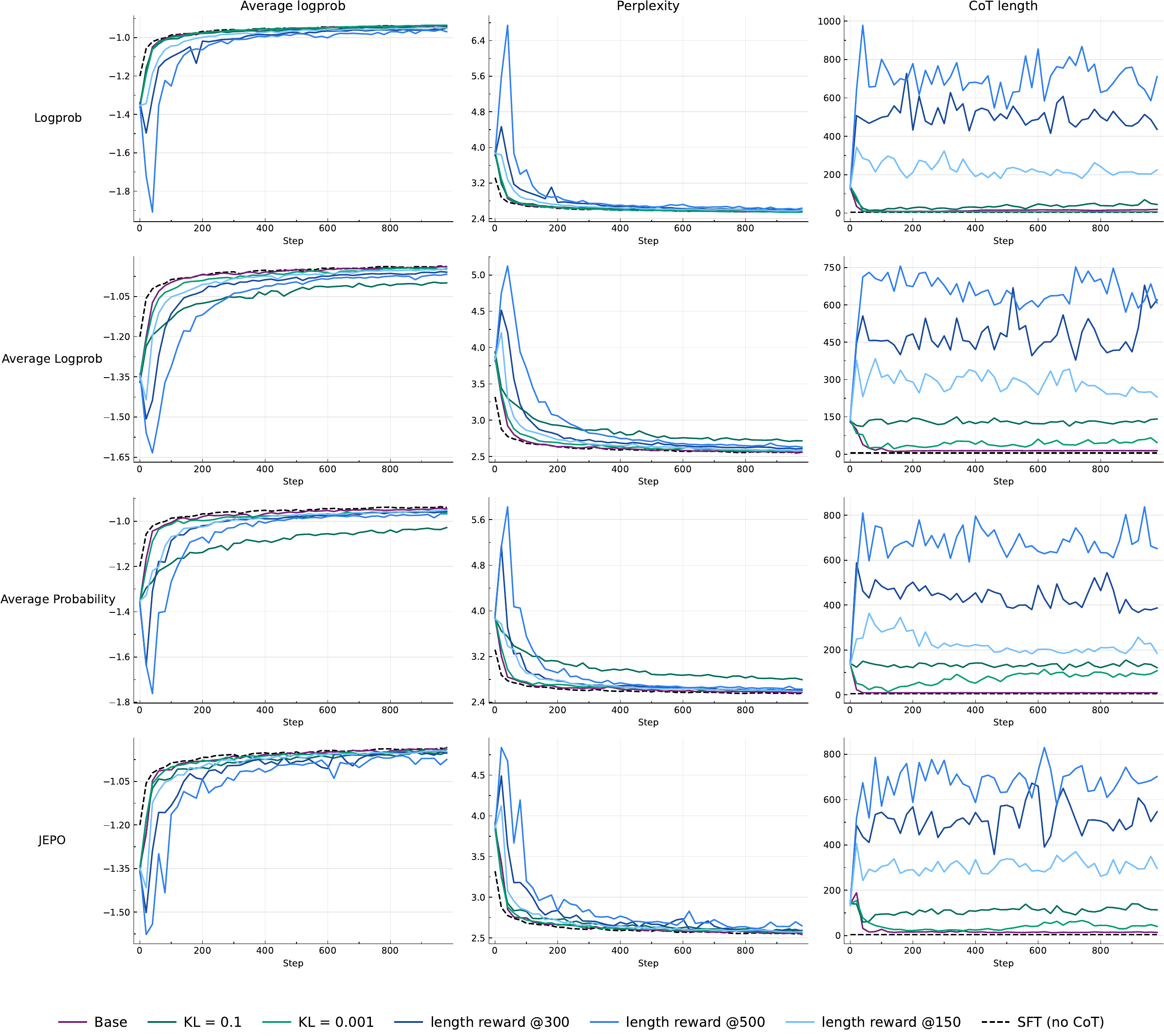}
    \caption{\small{\bf Non-verifiable. Llama 3.2 3B Instruct on Alpaca.} Conclusions are similar to \Cref{fig-app:ablations_l3b_numina}.}
    \label{fig-app:ablations_l3b_alpaca}
\end{figure}

\begin{figure}
    \centering
    \includegraphics[width=\linewidth]{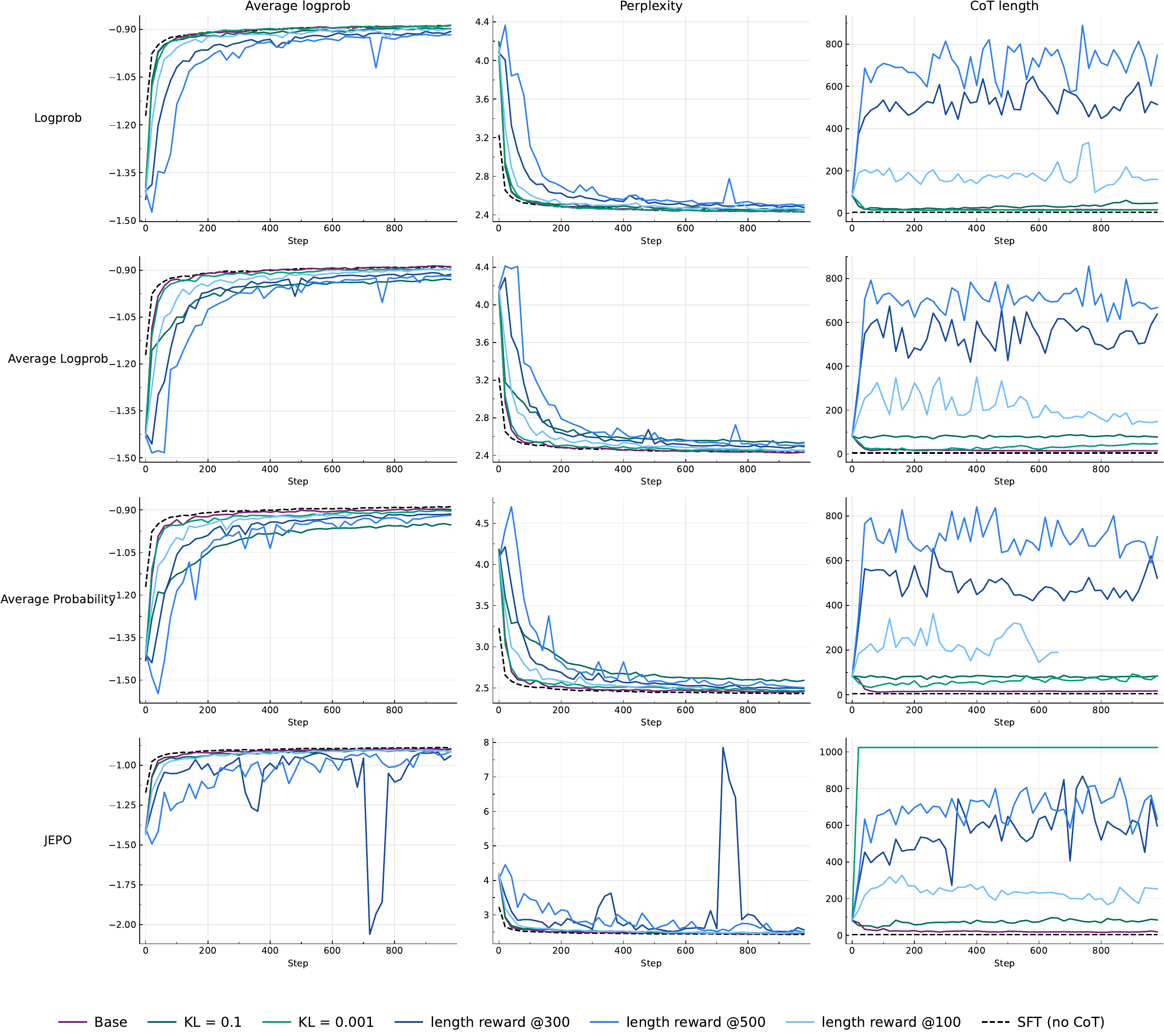}
    \caption{\small{\bf Non-verifiable. Qwen 2.5 3B Instruct on Alpaca.} Conclusions are similar to \Cref{fig-app:ablations_l3b_numina}.}
    \label{fig-app:ablations_q3b_alpaca}
\end{figure}

\newpage
\section{Impact of the Marginal Log-Probability Estimate}

We observe that in some cases, there is a significant difference between
the naive MC1 estimate of the true logprob of the reference answer, and the MC32 estimate. In particular, in verifiable domains where all algorithms reliably learn a nontrivial chain of thought, base RL and probability rewards exhibit a large difference between the two estimates. Due to the high computational cost of frequent MC32 evaluations, we visualize the results in \Cref{fig-app:l3b_ds_mlp,fig-app:q3b_ds_mlp} for a subset of our experimental settings.

\begin{figure}
    \centering
    \includegraphics[width=\linewidth]{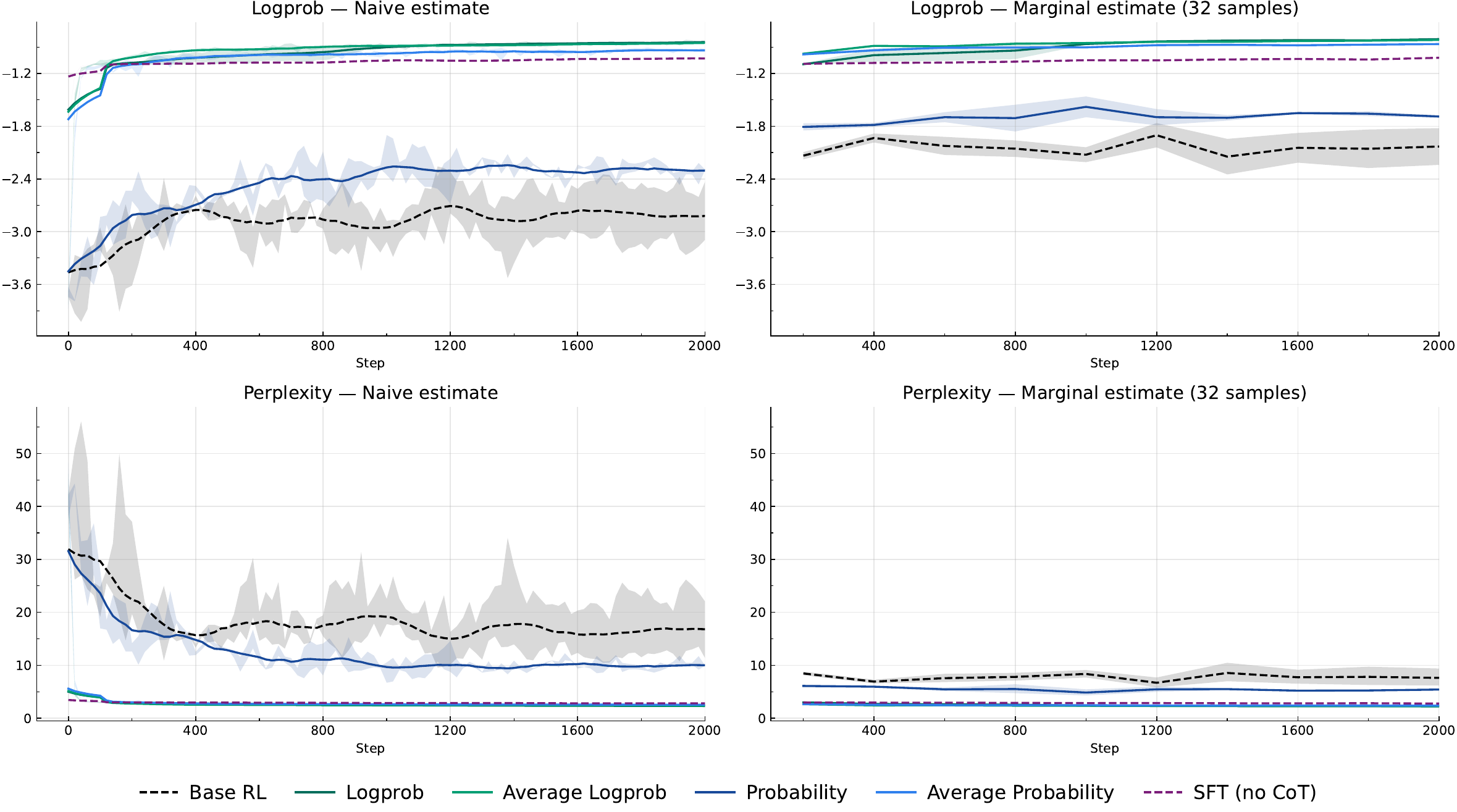}
    \caption{\small{\bf Verifiable. Llama 3.1 3B Instruct on DeepScaleR.} Logprob-style rewards (including average probability) achieve a good performance, beating the SFT baseline, keeping the difference between MC1 and MC32 estimates relatively small. In contrast, baseline RL and probability rewards perform poorly on these metrics.}
    \label{fig-app:l3b_ds_mlp}
\end{figure}

\begin{figure}
    \centering
    \includegraphics[width=\linewidth]{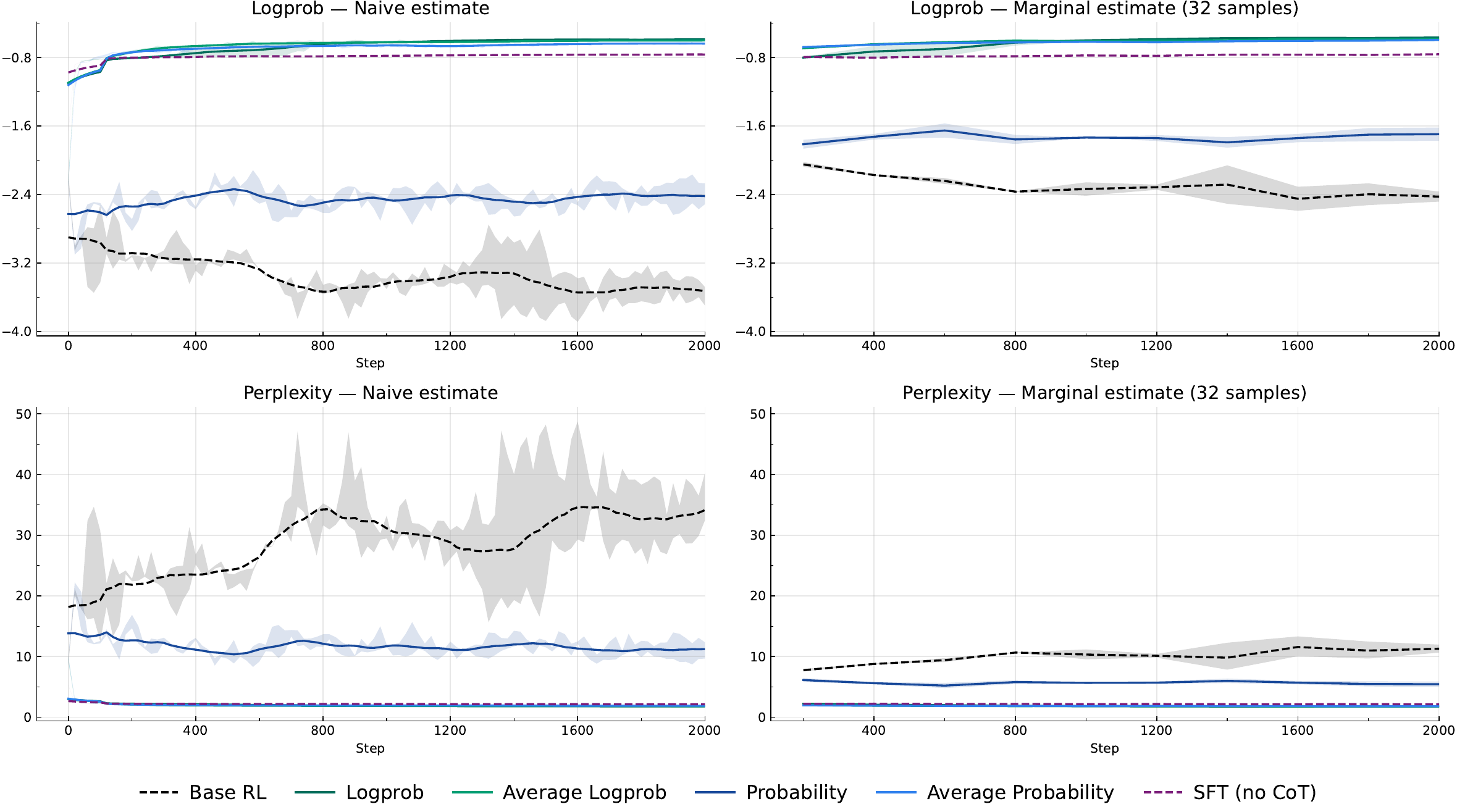}
    \caption{\small{\bf Verifiable. Qwen 2.5 3B Instruct on DeepScaleR.}
    The patterns are largely similar to the Llama model.}
    \label{fig-app:q3b_ds_mlp}
\end{figure}

\section{Correlation Analysis}
\label{sec:correlation}

\begin{figure}
    \centering
    \includegraphics[width=0.5\linewidth]{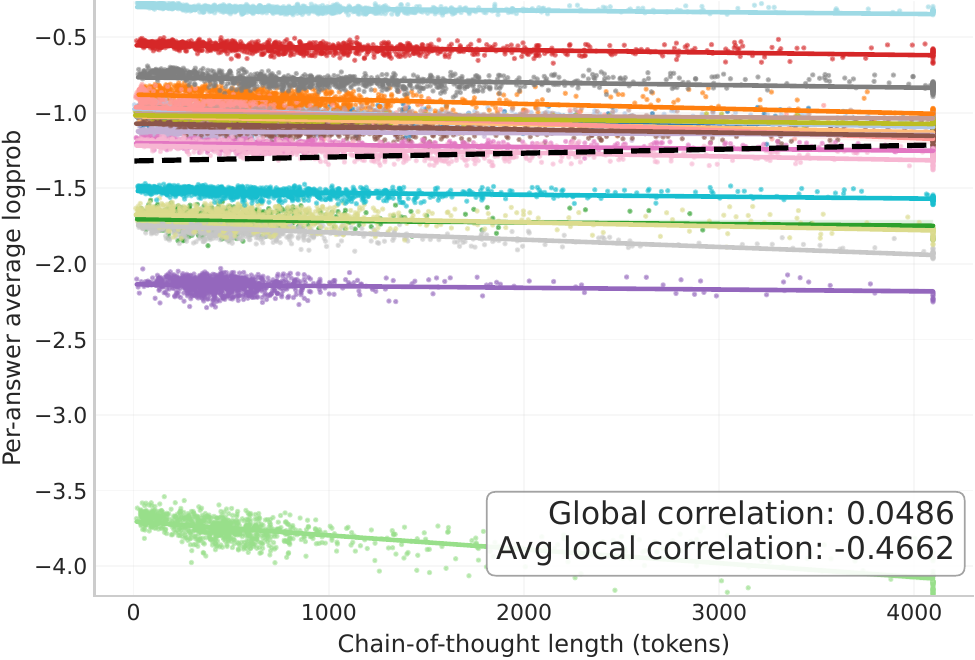}
    \caption{Global and local (per-question) correlations between the logprob of the correct answer and the CoT length, for the Llama 3B Instruct model, on the Numina dataset. Correlations are computed on a random selection of 100 questions from the dataset, each with 1000 CoTs generated with T=1, whereas the graph only includes 20 questions for visual clarity.}
    \label{fig-app:numina_corr_color}
\end{figure}

\begin{figure}
    \centering
    \includegraphics[width=0.5\linewidth]{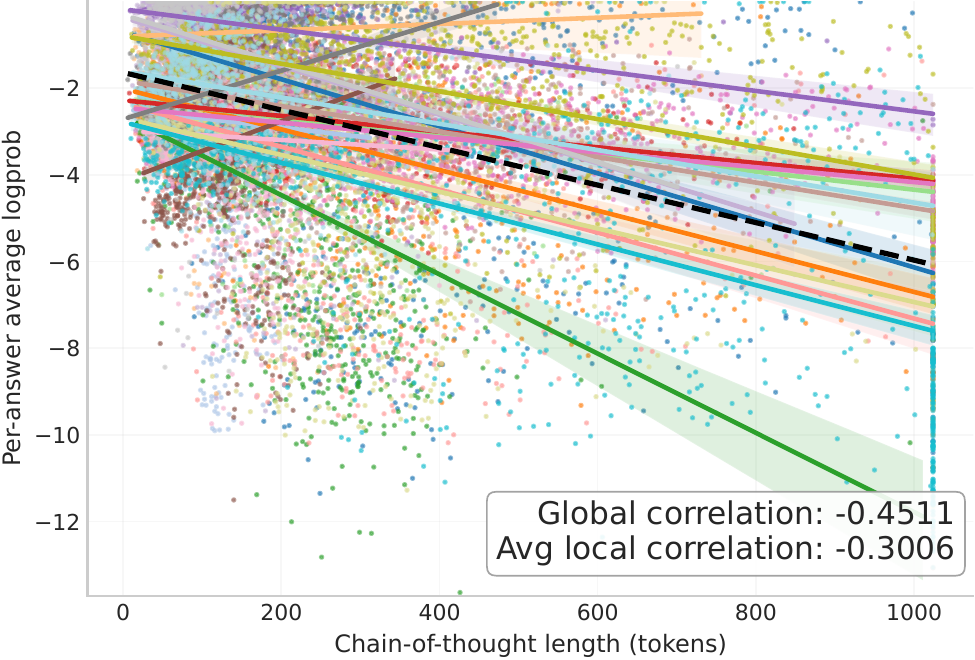}
    \caption{Global and local (per-question) correlations between the logprob of the correct answer and the CoT length, for the Llama 3B Instruct model, on the MATH dataset. Correlations are computed on a random selection of 100 questions from the dataset, each with 1000 CoTs generated with T=1, whereas the graph only includes 20 questions for visual clarity.}
    \label{fig-app:math_corr_color}
\end{figure}

\begin{figure}
    \centering
    \includegraphics[width=0.5\linewidth]{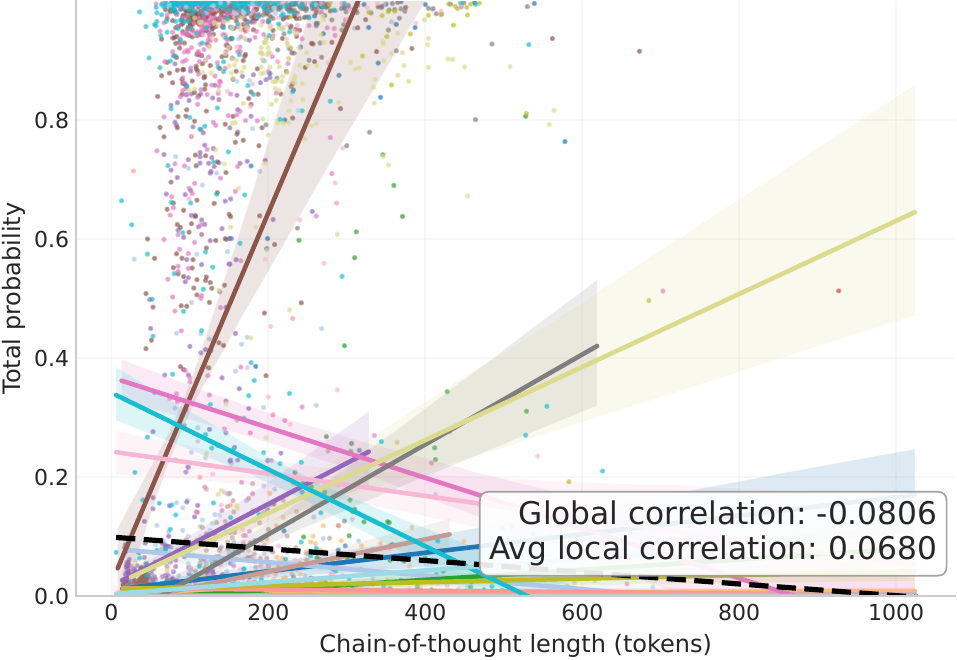}
    \caption{Global and local (per-question) correlations between the \emph{probability} of the correct answer and the CoT length, for the Llama 3B Instruct model, on the MATH dataset. Correlations are computed on a random selection of 100 questions from the dataset, each with 1000 CoTs generated with T=1, whereas the graph only includes 20 questions for visual clarity.}
    \label{fig-app:mathprob_corr_color}
\end{figure}

To investigate the drastic decrease in the CoT length, we measure two
metrics on the initial models' outputs. For 100 random problems from each dataset, we generate 1000 CoTs and measure the correlation between the CoT length, and the probability or logprobability of getting the correct answer after this CoT.

We report two variants of this metric. On the one hand, we have the
global correlation, which simply pools together all CoTs and measures the
Pearson correlation. On the other hand, we compute the ``average
\emph{local} correlation'', that is, we group the completions per-prompt,
compute the correlation for each prompt, and average those values over
prompts.

These two variants can lead to significantly different values,
similarly to Simpson's paradox, as can be seen below.

The key insight here is that when group-relative advantages are computed
using GRPO or RLOO, they are determined relative to other trajectories
for the same problem. So local correlations would be more impactful than
global correlations in driving down the CoT length during GRPO or RLOO.

We observe this in
\Cref{fig-app:numina_corr_color,fig-app:math_corr_color,fig-app:mathprob_corr_color}.
For logprobs on Numina, the global correlation is very low in absolute
terms, and slightly positive. However, the average local correlation is clearly negative, which corresponds to the regime in which the CoT rapidly drops at the beginning of the training. For logprobs on MATH, both the global and local correlations are negative, and indeed we observe a level of CoT degradation in that case. Conversely, when measuring the probability of the correct answer on MATH, the global correlation happens to be slightly negative, but the average local correlation is slightly positive -- indeed, in this case we do not observe a CoT collapse.

%% file: tables/verifiable_g4.tex
\begin{table*}[h]
\centering
\footnotesize
\setlength{\tabcolsep}{5pt}
\renewcommand{\arraystretch}{1.15}
\resizebox{\textwidth}{!}{%
\begin{tabular}{lcccccccc}
\toprule
 & Base model & Base RL & Log-prob & Avg Logprob & JEPO & Probability & Avg Probability & SFT (no CoT) \\
\midrule
\addlinespace\textbf{Llama 3B, MATH} &  &  &  &  &  &  &  &  \\
Greedy success & 17.13 ± 0.00 & 42.09 ± 1.08 & 40.58 ± 2.63 & 40.52 ± 0.01 & 43.19 ± 0.10 & 40.48 ± 0.42 & 41.56 ± 0.52 & 14.38 \\
T=1 Sampled Success & 10.77 ± 0.07 & 40.34 ± 1.30 & 31.81 ± 1.40 & 31.06 ± 1.82 & 35.85 ± 0.80 & 38.03 ± 0.99 & 35.03 ± 0.73 & 9.96 \\
Average log-prob MC32 & — & — & -0.72 ± 0.02 & -0.73 ± 0.02 & -0.79 ± 0.01 & -1.52 ± 0.04 & -0.82 ± 0.00 & -0.96 \\
Average log-prob & -4.21 ± 0.00 & -2.98 ± 0.06 & -0.83 ± 0.01 & -1.02 ± 0.01 & -1.23 ± 0.04 & -2.56 ± 0.06 & -1.11 ± 0.01 & -0.96 \\
Perplexity & 67.29 ± 0.00 & 19.65 ± 1.17 & 2.30 ± 0.03 & 2.77 ± 0.04 & 3.41 ± 0.14 & 12.97 ± 0.81 & 3.03 ± 0.04 & 2.62 \\
CoT length & 331.66 ± 2.02 & 373.52 ± 23.15 & 283.84 ± 27.47 & 305.76 ± 40.62 & 326.15 ± 21.00 & 345.83 ± 10.13 & 333.47 ± 45.09 & 5.00 \\
\addlinespace\textbf{Qwen 3B, MATH} &  &  &  &  &  &  &  &  \\
Greedy success & 21.19 ± 0.00 & 53.19 ± 0.18 & 51.70 ± 2.65 & 53.87 ± 0.24 & 54.01 ± 2.10 & 54.59 ± 0.24 & 53.52 ± 1.47 & 18.32 \\
T=1 Sampled Success & 16.10 ± 0.06 & 50.31 ± 0.03 & 36.36 ± 1.71 & 37.53 ± 1.32 & 42.55 ± 3.69 & 48.75 ± 0.77 & 42.37 ± 2.22 & 12.00 \\
Average log-prob MC32 & — & — & -0.47 ± 0.00 & -0.46 ± 0.01 & -0.47 ± 0.02 & -1.00 ± 0.03 & -0.49 ± 0.02 & — \\
Average log-prob & -2.19 ± 0.00 & -1.89 ± 0.01 & -0.67 ± 0.09 & -0.68 ± 0.01 & -0.78 ± 0.02 & -1.77 ± 0.02 & -0.79 ± 0.11 & -0.69 \\
Perplexity & 8.92 ± 0.00 & 6.60 ± 0.03 & 1.97 ± 0.19 & 1.98 ± 0.02 & 2.18 ± 0.05 & 5.85 ± 0.13 & 2.22 ± 0.24 & 1.99 \\
CoT length & 222.43 ± 0.80 & 429.31 ± 0.17 & 326.69 ± 49.32 & 387.34 ± 7.17 & 362.67 ± 70.73 & 447.37 ± 34.30 & 380.90 ± 10.57 & 5.00 \\
\addlinespace\textbf{Llama 3B, DeepScaleR} &  &  &  &  &  &  &  &  \\
Greedy success & 10.05 ± 0.00 & 26.18 ± 0.74 & 28.93 ± 1.55 & 29.67 ± 2.14 & 28.23 ± 0.11 & 27.01 ± 0.55 & 28.54 ± 0.89 & 9.88 \\
T=1 Sampled Success & 5.25 ± 0.92 & 21.05 ± 1.21 & 15.64 ± 0.73 & 16.33 ± 0.34 & 19.16 ± 0.59 & 23.54 ± 0.58 & 18.97 ± 0.32 & 5.14 \\
Average log-prob MC32 & — & -1.58 ± 0.03 & -0.84 ± 0.00 & -0.85 ± 0.00 & -0.85 ± 0.01 & -1.59 ± 0.02 & -0.89 ± 0.01 & -1.01 \\
Average log-prob & -4.92 ± 0.00 & -2.74 ± 0.04 & -1.02 ± 0.04 & -1.19 ± 0.15 & -1.73 ± 0.28 & -2.99 ± 0.07 & -1.97 ± 0.51 & -1.01 \\
Perplexity & 137.25 ± 0.00 & 15.49 ± 0.63 & 2.77 ± 0.10 & 3.30 ± 0.50 & 5.77 ± 1.60 & 19.90 ± 1.47 & 7.97 ± 4.41 & 2.74 \\
CoT length & 342.53 ± 1.31 & 317.31 ± 8.96 & 421.21 ± 27.22 & 378.21 ± 25.53 & 322.58 ± 13.60 & 349.44 ± 5.37 & 350.73 ± 20.70 & 5.00 \\
\addlinespace\textbf{Qwen 3B, DeepScaleR} &  &  &  &  &  &  &  &  \\
Greedy success & 14.02 ± 0.00 & 37.88 ± 2.16 & 36.21 ± 1.09 & 36.54 ± 0.72 & 37.04 ± 0.12 & 37.60 ± 0.00 & 35.67 ± 1.14 & 11.52 \\
T=1 Sampled Success & 10.05 ± 0.64 & 31.57 ± 2.99 & 22.34 ± 0.96 & 21.74 ± 0.95 & 27.12 ± 0.03 & 31.30 ± 0.32 & 25.69 ± 0.61 & 6.48 \\
Average log-prob MC32 & — & -1.61 ± 0.13 & -0.59 ± 0.01 & -0.61 ± 0.00 & -0.61 ± 0.00 & -1.38 ± 0.00 & -0.63 ± 0.00 & -0.76 \\
Average log-prob & -2.57 ± 0.00 & -2.63 ± 0.05 & -0.69 ± 0.02 & -0.88 ± 0.09 & -0.96 ± 0.00 & -2.49 ± 0.07 & -1.09 ± 0.07 & -0.76 \\
Perplexity & 13.12 ± 0.00 & 13.87 ± 0.75 & 2.00 ± 0.05 & 2.43 ± 0.22 & 2.62 ± 0.01 & 12.10 ± 0.84 & 2.99 ± 0.21 & 2.14 \\
CoT length & 216.25 ± 0.23 & 422.15 ± 65.98 & 429.51 ± 12.13 & 398.42 ± 30.20 & 436.50 ± 9.25 & 423.48 ± 5.80 & 396.55 ± 7.76 & 5.00 \\
\bottomrule
\end{tabular}
}
\caption{\small {\bf Results on verifiable domains, G=4.} Final performance of models trained with a group size of 4, across all our algorithms (including JEPO) and metrics. Conclusions mirror those of \Cref{tab:verifiable_g32}.  The corresponding learning curves are presented in \Cref{fig-app:llama_math_g4,fig-app:qwen_math_g4,fig-app:llama_ds_g4,fig-app:qwen_ds_g4}.}
\label{tab:verifiable_g4}
\end{table*}

%% file: neurips_2025.bib
@article{openai2023gpt4,
  title={GPT-4 Technical Report},
  author={OpenAI},
  journal={arXiv preprint arXiv:2303.08774},
  year={2023}
}

@article{wei2022chain,
  title={Chain-of-thought prompting elicits reasoning in large language models},
  author={Wei, Jason and Wang, Xuezhi and Schuurmans, Dale and Bosma, Maarten and Xia, Fei and Chi, Ed and Le, Quoc V and Zhou, Denny and others},
  journal={Advances in neural information processing systems},
  volume={35},
  pages={24824--24837},
  year={2022}
}

@article{schulman2017,
  title={Proximal policy optimization algorithms},
  author={Schulman, John and Wolski, Filip and Dhariwal, Prafulla and Radford, Alec and Klimov, Oleg},
  journal={arXiv preprint arXiv:1707.06347},
  year={2017}
}

@article{kingma2014adam,
  title={Adam: A method for stochastic optimization},
  author={Kingma, Diederik P and Ba, Jimmy},
  journal={arXiv preprint arXiv:1412.6980},
  year={2014}
}

@article{ahmadian2024back,
  title={Back to basics: Revisiting reinforce style optimization for learning from human feedback in llms},
  author={Ahmadian, Arash and Cremer, Chris and Gall{\'e}, Matthias and Fadaee, Marzieh and Kreutzer, Julia and Pietquin, Olivier and {\"U}st{\"u}n, Ahmet and Hooker, Sara},
  journal={arXiv preprint arXiv:2402.14740},
  year={2024}
}

@article{cobbe2021training,
  title={Training verifiers to solve math word problems},
  author={Cobbe, Karl and Kosaraju, Vineet and Bavarian, Mohammad and Chen, Mark and Jun, Heewoo and Kaiser, Lukasz and Plappert, Matthias and Tworek, Jerry and Hilton, Jacob and Nakano, Reiichiro and others},
  journal={arXiv preprint arXiv:2110.14168},
  year={2021}
}

@inproceedings{rloo,
    title = "Back to Basics: Revisiting {REINFORCE}-Style Optimization for Learning from Human Feedback in {LLM}s",
    author = {Ahmadian, Arash  and
      Cremer, Chris  and
      Gall{\'e}, Matthias  and
      Fadaee, Marzieh  and
      Kreutzer, Julia  and
      Pietquin, Olivier  and
      {\"U}st{\"u}n, Ahmet  and
      Hooker, Sara},
    editor = "Ku, Lun-Wei  and
      Martins, Andre  and
      Srikumar, Vivek",
    booktitle = "Proceedings of the 62nd Annual Meeting of the Association for Computational Linguistics (Volume 1: Long Papers)",
    month = aug,
    year = "2024",
    address = "Bangkok, Thailand",
    publisher = "Association for Computational Linguistics",
    url = "https://aclanthology.org/2024.acl-long.662/",
    doi = "10.18653/v1/2024.acl-long.662",
    pages = "12248--12267"
}

@inproceedings{
agarwal2025the,
title={The Unreasonable Effectiveness of Entropy Minimization in {LLM} Reasoning},
author={Shivam Agarwal and Zimin Zhang and Lifan Yuan and Jiawei Han and Hao Peng},
booktitle={The Thirty-ninth Annual Conference on Neural Information Processing Systems},
year={2025},
url={https://openreview.net/forum?id=UfFTBEsLgI}
}

@inproceedings{HendrycksBKABTS21,
  author={Dan Hendrycks and Collin Burns and Saurav Kadavath and Akul Arora and Steven Basart and Eric Tang and Dawn Song and Jacob Steinhardt},
  title={Measuring Mathematical Problem Solving With the MATH Dataset},
  year={2021},
  cdate={1609459200000},
  url={https://datasets-benchmarks-proceedings.neurips.cc/paper/2021/hash/be83ab3ecd0db773eb2dc1b0a17836a1-Abstract-round2.html},
  booktitle={NeurIPS Datasets and Benchmarks},
}

@inproceedings{
tang2025beyond,
title={Beyond Verifiable Rewards: Scaling Reinforcement Learning in Language Models to Unverifiable Data},
author={Yunhao Tang and Sid Wang and Lovish Madaan and Remi Munos},
booktitle={The Thirty-ninth Annual Conference on Neural Information Processing Systems},
year={2025},
url={https://openreview.net/forum?id=pc6M9h3T9m}
}


%% file: refs.bib
@misc{alpaca,
  author = {Rohan Taori and Ishaan Gulrajani and Tianyi Zhang and Yann Dubois and Xuechen Li and Carlos Guestrin and Percy Liang and Tatsunori B. Hashimoto },
  title = {Stanford Alpaca: An Instruction-following LLaMA model},
  year = {2023},
  publisher = {GitHub},
  journal = {GitHub repository},
  howpublished = {\url{https://github.com/tatsu-lab/stanford_alpaca}},
}

@misc{zhu2025surveylatentreasoning,
      title={A Survey on Latent Reasoning}, 
      author={Rui-Jie Zhu and Tianhao Peng and Tianhao Cheng and Xingwei Qu and Jinfa Huang and Dawei Zhu and Hao Wang and Kaiwen Xue and Xuanliang Zhang and Yong Shan and Tianle Cai and Taylor Kergan and Assel Kembay and Andrew Smith and Chenghua Lin and Binh Nguyen and Yuqi Pan and Yuhong Chou and Zefan Cai and Zhenhe Wu and Yongchi Zhao and Tianyu Liu and Jian Yang and Wangchunshu Zhou and Chujie Zheng and Chongxuan Li and Yuyin Zhou and Zhoujun Li and Zhaoxiang Zhang and Jiaheng Liu and Ge Zhang and Wenhao Huang and Jason Eshraghian},
      year={2025},
      eprint={2507.06203},
      archivePrefix={arXiv},
      primaryClass={cs.CL},
      url={https://arxiv.org/abs/2507.06203}, 
}

@misc{jayalath2025computeteacherturninginference,
      title={Compute as Teacher: Turning Inference Compute Into Reference-Free Supervision}, 
      author={Dulhan Jayalath and Shashwat Goel and Thomas Foster and Parag Jain and Suchin Gururangan and Cheng Zhang and Anirudh Goyal and Alan Schelten},
      year={2025},
      eprint={2509.14234},
      archivePrefix={arXiv},
      primaryClass={cs.LG},
      url={https://arxiv.org/abs/2509.14234}, 
}

@misc{simonds2025rlsrreinforcementlearningself,
      title={RLSR: Reinforcement Learning from Self Reward}, 
      author={Toby Simonds and Kevin Lopez and Akira Yoshiyama and Dominique Garmier},
      year={2025},
      eprint={2505.08827},
      archivePrefix={arXiv},
      primaryClass={cs.LG},
      url={https://arxiv.org/abs/2505.08827}, 
}

@inproceedings{lee2024rlaif,
  title={RLAIF vs. RLHF: Scaling Reinforcement Learning from Human Feedback with AI Feedback},
  author={Lee, Harrison and Phatale, Samrat and Mansoor, Hassan and Mesnard, Thomas and Ferret, Johan and Lu, Kellie Ren and Bishop, Colton and Hall, Ethan and Carbune, Victor and Rastogi, Abhinav and others},
  booktitle={International Conference on Machine Learning},
  pages={26874--26901},
  year={2024},
  organization={PMLR}
}

@article{whitehouse2025j1,
  title={J1: Incentivizing thinking in llm-as-a-judge via reinforcement learning},
  author={Whitehouse, Chenxi and Wang, Tianlu and Yu, Ping and Li, Xian and Weston, Jason and Kulikov, Ilia and Saha, Swarnadeep},
  journal={arXiv preprint arXiv:2505.10320},
  year={2025}
}

@misc{bai2022constitutionalaiharmlessnessai,
      title={Constitutional AI: Harmlessness from AI Feedback}, 
      author={Yuntao Bai and Saurav Kadavath and Sandipan Kundu and Amanda Askell and Jackson Kernion and Andy Jones and Anna Chen and Anna Goldie and Azalia Mirhoseini and Cameron McKinnon and Carol Chen and Catherine Olsson and Christopher Olah and Danny Hernandez and Dawn Drain and Deep Ganguli and Dustin Li and Eli Tran-Johnson and Ethan Perez and Jamie Kerr and Jared Mueller and Jeffrey Ladish and Joshua Landau and Kamal Ndousse and Kamile Lukosuite and Liane Lovitt and Michael Sellitto and Nelson Elhage and Nicholas Schiefer and Noemi Mercado and Nova DasSarma and Robert Lasenby and Robin Larson and Sam Ringer and Scott Johnston and Shauna Kravec and Sheer El Showk and Stanislav Fort and Tamera Lanham and Timothy Telleen-Lawton and Tom Conerly and Tom Henighan and Tristan Hume and Samuel R. Bowman and Zac Hatfield-Dodds and Ben Mann and Dario Amodei and Nicholas Joseph and Sam McCandlish and Tom Brown and Jared Kaplan},
      year={2022},
      eprint={2212.08073},
      archivePrefix={arXiv},
      primaryClass={cs.CL},
      url={https://arxiv.org/abs/2212.08073}, 
}

@misc{numina_math_datasets,
  author = {Jia Li and Edward Beeching and Lewis Tunstall and Ben Lipkin and Roman Soletskyi and Shengyi Costa Huang and Kashif Rasul and Longhui Yu and Albert Jiang and Ziju Shen and Zihan Qin and Bin Dong and Li Zhou and Yann Fleureau and Guillaume Lample and Stanislas Polu},
  title = {NuminaMath},
  year = {2024},
  publisher = {Numina},
  journal = {Hugging Face repository},
  howpublished = {\url{[https://huggingface.co/AI-MO/NuminaMath-CoT](https://github.com/project-numina/aimo-progress-prize/blob/main/report/numina_dataset.pdf)}}
}

@inproceedings{huang2025sharpening,
  title        = {Self-Improvement in Language Models: The Sharpening Mechanism},
  author       = {Audrey Huang and Adam Block and Dylan J. Foster and Dhruv Rohatgi and Cyril Zhang and Max Simchowitz and Jordan T. Ash and Akshay Krishnamurthy},
  booktitle    = {ICLR},
  year         = {2025},
  howpublished = {OpenReview preprint},
  note         = {arXiv:2412.01951},
  url          = {https://openreview.net/forum?id=WJaUkwci9o}
}

@misc{liu2025noverincentivetraininglanguage,
      title={NOVER: Incentive Training for Language Models via Verifier-Free Reinforcement Learning}, 
      author={Wei Liu and Siya Qi and Xinyu Wang and Chen Qian and Yali Du and Yulan He},
      year={2025},
      eprint={2505.16022},
      archivePrefix={arXiv},
      primaryClass={cs.CL},
      url={https://arxiv.org/abs/2505.16022}, 
}

@inproceedings{
gurung2025learning,
title={Learning to Reason for Long-Form Story Generation},
author={Alexander Gurung and Mirella Lapata},
booktitle={Second Conference on Language Modeling},
year={2025},
url={https://openreview.net/forum?id=dr3eg5ehR2}
}

@article{gao2025one,
  title={One-shot entropy minimization},
  author={Gao, Zitian and Chen, Lynx and Luo, Haoming and Zhou, Joey and Dai, Bryan},
  journal={arXiv preprint arXiv:2505.20282},
  year={2025}
}

@article{li2025darling,
  title        = {Jointly Reinforcing Diversity and Quality in Language Model Generations (DARLING)},
  author       = {Tianjian Li and Yiming Zhang and Ping Yu and Swarnadeep Saha and Daniel Khashabi and Jason Weston and Jack Lanchantin and Tianlu Wang},
  year         = {2025},
  archivePrefix= {arXiv},
  eprint       = {2509.02534},
  primaryClass = {cs.LG},
  url          = {https://arxiv.org/abs/2509.02534}
}

@inproceedings{huang2025bestofn,
  title     = {Is Best-of-N the Best of Them? Coverage, Scaling, and Optimality in Inference-Time Alignment},
  author    = {Huang, Audrey and Block, Adam and Liu, Qinghua and Jiang, Nan and Krishnamurthy, Akshay and Foster, Dylan J.},
  booktitle = {International Conference on Machine Learning (ICML)},
  year      = {2025},
  url       = {https://arxiv.org/abs/2503.21878},
  note      = {See also OpenReview: QnjfkhrbYK}
}

@article{song2025outcomebased,
  title        = {Outcome-based Exploration for LLM Reasoning},
  author       = {Yuda Song and Julia Kempe and R\'emi Munos},
  year         = {2025},
  eprint       = {2509.06941},
  archivePrefix= {arXiv},
  primaryClass = {cs.LG},
  url          = {https://arxiv.org/abs/2509.06941}
}

@article{kayal2025intrinsic,
  title        = {The Impact of Intrinsic Rewards on Exploration in Reinforcement Learning},
  author       = {Aya sur Kayal and Eduardo Pignatelli and Laura Toni},
  year         = {2025},
  eprint       = {2501.11533},
  archivePrefix= {arXiv},
  primaryClass = {cs.LG},
  url          = {https://arxiv.org/abs/2501.11533}
}

@misc{zhao2025learningreasonexternalrewards,
      title={Learning to Reason without External Rewards}, 
      author={Xuandong Zhao and Zhewei Kang and Aosong Feng and Sergey Levine and Dawn Song},
      year={2025},
      eprint={2505.19590},
      archivePrefix={arXiv},
      primaryClass={cs.LG},
      url={https://arxiv.org/abs/2505.19590}, 
}

@misc{prabhudesai2025maximizingconfidenceimprovesreasoning,
      title={Maximizing Confidence Alone Improves Reasoning}, 
      author={Mihir Prabhudesai and Lili Chen and Alex Ippoliti and Katerina Fragkiadaki and Hao Liu and Deepak Pathak},
      year={2025},
      eprint={2505.22660},
      archivePrefix={arXiv},
      primaryClass={cs.LG},
      url={https://arxiv.org/abs/2505.22660}, 
}

@misc{li2025confidenceneedfewshotrl,
      title={Confidence Is All You Need: Few-Shot RL Fine-Tuning of Language Models}, 
      author={Pengyi Li and Matvey Skripkin and Alexander Zubrey and Andrey Kuznetsov and Ivan Oseledets},
      year={2025},
      eprint={2506.06395},
      archivePrefix={arXiv},
      primaryClass={cs.CL},
      url={https://arxiv.org/abs/2506.06395}, 
}

@article{chen2021evaluating,
  title={Evaluating large language models trained on code},
  author={Chen, Mark and Tworek, Jerry and Jun, Heewoo and Yuan, Qiming and de Oliveira Pinto, Henrique Ponde and Kaplan, Jared and Edwards, Harri and Burda, Yuri and Joseph, Nicholas and Brockman, Greg and others},
  journal={arXiv preprint arXiv:2107.03374},
  year={2021}
}

@inproceedings{austin2021programs,
  title={Program synthesis with large language models},
  author={Austin, Jacob and Odena, Augustus and Nye, Maxwell and Bosma, Maarten and Michalewski, Henryk and Dohan, David and Jiang, Ellen and Cai, Carrie and Terry, Michael and Le, Quoc and others},
  booktitle={ICML},
  year={2021}
}

@inproceedings{
hendrycks2021measuring_apps,
title={Measuring Coding Challenge Competence With {APPS}},
author={Dan Hendrycks and Steven Basart and Saurav Kadavath and Mantas Mazeika and Akul Arora and Ethan Guo and Collin Burns and Samir Puranik and Horace He and Dawn Song and Jacob Steinhardt},
booktitle={Thirty-fifth Conference on Neural Information Processing Systems Datasets and Benchmarks Track (Round 2)},
year={2021},
url={https://openreview.net/forum?id=sD93GOzH3i5}
}

@misc{deepscaler2025,
  title={DeepScaleR: Surpassing O1-Preview with a 1.5B Model by Scaling RL},
  author={Michael Luo and Sijun Tan and Justin Wong and Xiaoxiang Shi and William Tang and Manan Roongta and Colin Cai and Jeffrey Luo and Tianjun Zhang and Erran Li and Raluca Ada Popa and Ion Stoica},
  year={2025},
  howpublished={\url{https://pretty-radio-b75.notion.site/DeepScaleR-Surpassing-O1-Preview-with-a-1-5B-Model-by-Scaling-RL-19681902c1468005bed8ca303013a4e2}},
  note={Notion Blog},
}

@misc{dong2025reinforcementpretraining,
      title={Reinforcement Pre-Training}, 
      author={Qingxiu Dong and Li Dong and Yao Tang and Tianzhu Ye and Yutao Sun and Zhifang Sui and Furu Wei},
      year={2025},
      eprint={2506.08007},
      archivePrefix={arXiv},
      primaryClass={cs.CL},
      url={https://arxiv.org/abs/2506.08007}, 
}

@misc{yu2025rlpr,
  title        = {RLPR: Extrapolating RLVR to General Domains without Verifiers},
  author       = {Tianyu Yu and Bo Ji and Shouli Wang and Shu Yao and Zefan Wang and Ganqu Cui and Lifan Yuan and Ning Ding and Yuan Yao and Zhiyuan Liu and Maosong Sun and Tat-Seng Chua},
  year         = {2025},
  eprint       = {2506.18254},
  archivePrefix= {arXiv},
  primaryClass = {cs.LG},
  url          = {https://arxiv.org/abs/2506.18254}
}

@misc{zhou2025verifree,
  title        = {Reinforcing General Reasoning without Verifiers},
  author       = {Xiangxin Zhou and Zichen Liu and Anya Sims and Haonan Wang and Tianyu Pang and Chongxuan Li and Liang Wang and Min Lin and Chao Du},
  year         = {2025},
  howpublished = {arXiv preprint arXiv:2505.21493},
  note         = {Available at \url{https://arxiv.org/abs/2505.21493}, submitted May 27, 2025},
}

@article{song2024mind,
  title={Mind the gap: Examining the self-improvement capabilities of large language models},
  author={Song, Yuda and Zhang, Hanlin and Eisenach, Carson and Kakade, Sham and Foster, Dean and Ghai, Udaya},
  journal={arXiv preprint arXiv:2412.02674},
  year={2024}
}

@article{guo2025deepseek,
  title={Deepseek-r1: Incentivizing reasoning capability in llms via reinforcement learning},
  author={Guo, Daya and Yang, Dejian and Zhang, Haowei and Song, Junxiao and Zhang, Ruoyu and Xu, Runxin and Zhu, Qihao and Ma, Shirong and Wang, Peiyi and Bi, Xiao and others},
  journal={arXiv preprint arXiv:2501.12948},
  year={2025}
}

@article{Yang2024Qwen25TR,
  title={Qwen2.5 Technical Report},
  author={Qwen An Yang and Baosong Yang and Beichen Zhang and Binyuan Hui and Bo Zheng and Bowen Yu and Chengyuan Li and Dayiheng Liu and Fei Huang and Guanting Dong and Haoran Wei and Huan Lin and Jian Yang and Jianhong Tu and Jianwei Zhang and Jianxin Yang and Jiaxin Yang and Jingren Zhou and Junyang Lin and Kai Dang and Keming Lu and Keqin Bao and Kexin Yang and Le Yu and Mei Li and Mingfeng Xue and Pei Zhang and Qin Zhu and Rui Men and Runji Lin and Tianhao Li and Tingyu Xia and Xingzhang Ren and Xuancheng Ren and Yang Fan and Yang Su and Yi-Chao Zhang and Yunyang Wan and Yuqi Liu and Zeyu Cui and Zhenru Zhang and Zihan Qiu and Shanghaoran Quan and Zekun Wang},
  journal={ArXiv},
  year={2024},
  volume={abs/2412.15115},
}

@article{dubey2024llama,
  title={The llama 3 herd of models},
  author={Dubey, Abhimanyu and Jauhri, Abhinav and Pandey, Abhinav and Kadian, Abhishek and Al-Dahle, Ahmad and Letman, Aiesha and Mathur, Akhil and Schelten, Alan and Yang, Amy and Fan, Angela and others},
  journal={arXiv e-prints},
  pages={arXiv--2407},
  year={2024}
}
